\documentclass[final,1p,times]{elsarticle}

\usepackage{epsfig}
\usepackage{amsfonts}
\usepackage{amsmath}
\usepackage{cases}
\usepackage{amssymb}
\usepackage{ntheorem}
\usepackage{graphics}
\usepackage{graphicx}
\usepackage{mathrsfs}
\usepackage{indentfirst}
\usepackage{stmaryrd}
\usepackage{latexsym}
\usepackage{xypic}
\usepackage{epstopdf}
\usepackage{multirow}
\usepackage{booktabs}
\usepackage{color}
\biboptions{sort&compress}

\newtheorem{definition}{Definition}[section]
\newtheorem{lemma}[definition]{Lemma}
\newtheorem{theorem}[definition]{Theorem}

\newtheorem{assumption}[definition]{Assumption}

\newtheorem*{proof}{Proof}

\journal{XXX}

\begin{document}
	\begin{frontmatter}
		
		
		
		\title{\textbf{ERD: Exponential Retinex decomposition based on weak space and hybrid nonconvex regularization and its denoising application}}
		
		
		\author{Liang Wu$^{a}$, Wenjing Lu$^a$, Liming Tang$^{a,*}$, Zhuang Fang$^a$}
		\cortext[cor1]{Corresponding author. \\
			Email addresses: tlmcs78@foxmail.com }
		
		\address[A]{School of Mathematics and Statistics, Hubei Minzu University, Enshi, Hubei, 445000, P.R. China}
		
\begin{abstract}

The Retinex theory models the image as a product of illumination and reflection components, which has received extensive attention and is widely used in image enhancement, segmentation and color restoration. However, it has been rarely used in additive noise removal due to the inclusion of both multiplication and addition operations in the Retinex noisy image modeling.
In this paper, we propose an exponential Retinex decomposition model based on hybrid non-convex regularization and weak space oscillation-modeling for image denoising. The proposed model utilizes non-convex first-order total variation (TV) and non-convex second-order TV to regularize the reflection component and the illumination component, respectively, and employs weak $H^{-1}$ norm to measure the residual component.
By utilizing different regularizers, the proposed model effectively decomposes the image into reflection, illumination, and noise components. An alternating direction multipliers method (ADMM) combined with the Majorize-Minimization (MM) algorithm is developed to solve the proposed model. Furthermore, we provide a detailed proof of the convergence property of the algorithm.
Numerical experiments validate both the proposed model and algorithm. Compared with several state-of-the-art denoising models, the proposed model exhibits superior performance in terms of peak signal-to-noise ratio (PSNR) and mean structural similarity (MSSIM).
					
\end{abstract}

\begin{keyword}
image denoising, Retinex decomposition, weak space, nonconvex, regularization
\end{keyword}

\end{frontmatter}

\section{Introduction}

Image decomposition aims to decompose a given image into several unrelated components with intrinsic features, and has been widely used in image denoising, image enhancement and pattern recognition.  Various decomposition techniques have been proposed, among which the variational image model based on regularization and functional minimization is an effective method for solving this ill-posed  inverse problems. This mathematical model can be expressed as follows \cite{b1,b2,b3}:
\[\mathop {\min }\limits_{u,v} \mathcal{R}_1\left( u \right) + \mathcal{R}_2\left( v \right),\;\;\;s.t.\;f = u + v,\tag{1.1}\]
where $f$ is the input image, ${\mathcal{R}_{\rm{1}}}\left(  \cdot  \right)$ and ${\mathcal{R}_{\rm{2}}}\left(  \cdot  \right)$ are the regularization terms, which are used to promote the prior information in image components $u$ and $v$, respectively.
By solving the problem (1.1), we can obtain an effective decomposition of the image into its desired components. This variational approach has shown promising results in various applications, such as image denoising, which decomposes the image into clear image and noise component \cite{b4,b5,b6}; image structure-texture decomposition, which decomposes the image into cartoon component and oscillation component \cite{b7,b8,b9}; Retinex decomposition decomposes the image into illumination component and reflection component \cite{b10,b11,b12}.

The Retinex theory, originally proposed by Land and McCann in 1971 \cite{b13}, has a valuable property \cite{b14,b15}:  the reflection component represents the inherent color information of the object itself, which may contain high contrast, while the illumination component represents the influence of external light sources in the environment, which is usually a smooth field. Based on this property, Kimmel et al. \cite{b16} utilized the norm to impose constraints on reflection and illumination in bounded variation (BV) space and $H^1$ space, respectively, and proposed an image decomposition model based on variational Retinex.
In order to more effectively separate images with rich illumination, Ng and Wang \cite{b17} proposed a new Retinex image decomposition model (also called TVH1), which simultaneously minimizes illumination and reflection by combining more prior constraints. And then, Liang and Zhang \cite{b18} modified some regularization constraints in TVH1 model and proposed a high-order total variation $L^1$ decomposition model that uses a second-order total variation regularizer to characterize the illumination for the first time, and this model can flexibly select the constraint domain of the components according to different application scenarios. Recently, Xu et al. designed an image structure and texture estimator in \cite{b19} and proposed a new structure and texture aware Retinex (STAR) model for image decomposition.

However, the aforementioned Retinex models may not be well-suited for noisy images, as the decomposition quality tends to decrease significantly with increasing noise levels. To address this issue and improve the decomposition quality of noisy images, Liu et al. \cite{b20} proposed an exponential Retinex total variation (ETV) model, which has demonstrated satisfactory results in both Retinex decomposition and denoising applications. Building upon the ETV model, Wang et al. \cite{b21} introduced a nonconvex exponential total variation (NETV) model that utilizes a nonconvex norm to measure both the illumination and reflection components.

Traditionally, the $L^2$ norm is commonly employed as the fidelity term to quantify oscillation information. However, Meyer pointed out  in \cite{b22} that the $L^2$ norm may not be the best choice for measuring oscillation information since oscillation functions tend to have large $L^2$ norms. Consequently, it becomes challenging to effectively separate the oscillation component from high-frequency components \cite{b23,b24,b25}. With this fact, Meyer proposed three weak function spaces, $G$, $E$, and $F$, to accurately measure oscillation functions. Subsequently, numerous researchers delved into these function spaces and developed optimization algorithms for computing norms defined within these spaces. For example,  Osher et al. \cite{b26} leveraged the dual space $H^{-1}$ of $H^1$ as an approximation for the $G$ space and presented an efficient approach to simplify the $H^{-1}$ norm. This advancement significantly improved algorithmic accuracy when working with these function spaces.

In this paper, we propose an exponential Retinex image decomposition model for image denoising, based on the Retinex decomposition of noisy images. This new model combines oscillatory component weaker-norm modeling with the generalized nonconvex hybrid regularization technique.
The proposed decomposition model separates the noise image into three components: oscillation, reflection, and illumination. The oscillation information is measured using the $H^{-1}$ weak norm, while the reflection component is measured using the generalized non-convex first-order TV regularizer, which captures piecewise constant features. The illumination component, on the other hand, is measured using the generalized non-convex second-order TV regularizer, which captures spatial smoothness.
By incorporating these regularizers, our model effectively separates the oscillating noise component from the denoised image. The denoised image is then reconstructed using the reflection and illumination components. To solve this proposed model, we employ the ADMM   algorithm \cite{b27,b28}. Specifically, we utilize the Majorization-Minimization (MM) algorithm \cite{b29,b30} to solve two generalized nonconvex subproblems.
Furthermore, under reasonable assumptions, we rigorously prove the convergence of our proposed algorithm. In summary, the main contributions of this paper are as follows:
\begin{itemize}
\item[$\bullet$] An exponential Retinex image decomposition model based on weak space and hybrid  nonconvex regularization is proposed, which  effectively decomposes the noisy image into three components: oscillation, reflection, and illumination.

\item[$\bullet$] An ADMM  combined with the MM algorithm is introduced to solve the proposed model, and the sufficient conditions for the convergence of the proposed algorithm are provided.

\item[$\bullet$] We apply the proposed decomposition model to image denoising, where the restored image is reconstructed by  reflection and illumination components. The effectiveness and superiority of the proposed model and algorithm are validated through numerical experimental results.
\end{itemize}

The rest of this paper is organized as follows. In Section 2, we introduce the Retinex theory and show the  MM algorithm for solving non-convex minimization problems in detail. In Section 3, we propose a image exponential Retinex decomposition model based on weak space and hybrid nonconvex regularization, and introduce a numerical algorithm for the proposed model. In addition, the convergence conditions of the proposed algorithm are provided in this section. The
numerical experimental results showing the performance of the proposed
model are given in Section 4. Finally, we conclude this paper in Section 5.

\section{Background knowledge}
	
In this section, we provide a concise overview of the Retinex theory and the MM algorithm, which are relevant to our current research.

\subsection{Retinex theory}
	
Retinex theory was proposed by Land and McCann \cite{b13} in 1971, which is based on color perception modeling of human visual system. The Retinex theory holds that the color of the object itself is independent of the illumination distribution and intensity on the surface of the object, but is related to the reflection ability of different wavelengths of light. Based on this, the imaging process of an observed image can be represented by the following mathematical model,

\[O = I \odot R,
\tag{2.1}
\]
where $\odot$ denotes entrywise multiplication, $O$, $I$ and $R$ represent the perceived image, illumination and reflection, respectively. In general, we assume that $R \in \left[ {0,1} \right]$ is a piecewise constant region and $I \in \left[ {0, + \infty } \right]$ is a spatial smoothing function.

In order to effectively separate $I$ and $R$ from the above ill-posed problem, a commonly used approach is to eliminate the coupling relationship between $I$ and $R$ in problem (2.1) by logarithmic transformation.

\[o = i + r,
\tag{2.2}
\]
where $o = \log \left( O \right)$, $i = \log \left( I \right)$ and $r = \log \left( R \right)$.

Most Retinex variational models are based on Eq.(2.2). The logarithmic operation converts the multiplicative model into an ordinary additive model, effectively reducing the complexity of the algorithm and saving time costs. Additionally, the additive model facilitates discussions on algorithm convergence under certain conditions.

\subsection{MM algorithm for solving a nonconvex problem}

When facing a challenging optimization problem that is difficult to solve, it is often necessary to find an appropriate alternative function. Subsequently, the MM algorithm \cite{b29,b30} is utilized to convert the original optimization problem into an alternative function optimization problem.

To provide a clearer illustration of the MM algorithmic process, let's consider a general nonconvex optimization problem \cite{b31,b32} in image processing  as
\[\mathop {\min }\limits_u \left\{ {R\left( u \right): = \frac{\lambda }{2}||f - u||_2^2 + \int_\Omega  {\varphi \left( {u} \right)dx} } \right\}, \tag{2.3}\]
where $\Omega  \subseteq {R^2}$ is an image domain, $\lambda$ is a data fidelity term coefficient, and $\varphi \left(  \cdot  \right)$ is a concave potential function which usually has three choices as follows, (a) ${\varphi _1}\left( {t} \right) = |t{|^p}\left( {0 < p < 1} \right)$, (b) ${\varphi _2}\left( {t} \right) = \ln \left( {1 + \alpha |t|} \right)\left( {\alpha  > 0} \right)$, and (c) ${\varphi _3}\left( {t} \right) = \beta |t|/\left( {1 + \beta |t|} \right)\left( {\beta  > 0} \right)$.

In order to find an alternative function that meets the constraints of the MM algorithm, the following inequality is obtained by the properties of the concave function, \[\int_\Omega  {\varphi \left( u \right)dx}  \le \int_\Omega  {\varphi \left( v \right)dx}  + \int_\Omega  {\nabla \varphi \left( v \right)\left( {u - v} \right)dx} .\]
Since the non-differentiability of the indicator function $\varphi \left(  \cdot  \right)$ at zero, we define the following formula to replace the differential in the numerical algorithm, i.e.,
\[\nabla \varphi \left( v \right) = \left\{ \begin{array}{l}
\varphi '\left( {{v_\varepsilon }} \right),\;\;\;\;\;\;v \ge 0,\\
- \varphi '\left( {{v_\varepsilon }} \right),\;\;\;\;v < 0.
\end{array} \right.\tag{2.4}\]
where ${v_\varepsilon } = \left| v \right| + \varepsilon $, and $\varepsilon $ is a very small positive value, which is set at $\varepsilon {\rm{ = 1}} \times {\rm{1}}{{\rm{0}}^{ - 5}}$ in this paper. And then, we mark
\[T\left( {v,u} \right) = \frac{\lambda }{2}||f - u||_2^2 + \int_\Omega  {\varphi \left( {v} \right)dx}  + \int_\Omega  {\nabla \varphi \left( {v} \right)\left( {u - v} \right)dx} \tag{2.5}\]
and can easily prove that the following two conditional relations hold, i.e., (i)$T\left( {v,u} \right) \ge R\left( u \right)$, (ii)$T\left( {v,v} \right) = R\left( v \right)$, $\forall v \in \Omega $. Obviously, $T\left( {v,u} \right)$ is the substitution function that we have found. According to the MM algorithm, we let $v = {u^k}$, and the original minimization problem can be approximately replaced by
\[{u^{k + 1}} = \mathop {\arg \min }\limits_u T\left( {{u^k},u} \right) = \frac{\lambda }{2}||f - u||_2^2 + \int_\Omega  {\nabla \varphi \left( {{u^k}} \right)udx}  + C,\tag{2.6}\]
where $C$ is a constant that does not depend on $u$, which can be omitted in the minimization problem. The minimization problem (2.6) is a classical ${l^2}$
smooth problem, and we can readily obtain its analytic solution by the necessary condition of functional extremum,
\[{u^{k{\rm{ + 1}}}}{\rm{ = }}f - \frac{{\nabla \varphi \left( {{u^k}} \right)}}{\lambda }.\tag{2.7}\]
Similarly, the differential operator $\nabla \varphi \left( {{u^k}} \right)$ is defined as
\[\nabla \varphi \left( {{u^k}} \right) = \left\{ \begin{array}{l}
\varphi '\left( {u_\varepsilon ^k} \right),\;\;\;\;\;\;{u^k} \ge 0,\\
- \varphi '\left( {u_\varepsilon ^k} \right),\;\;\;\;{u^k} < 0.
\end{array} \right.\tag{2.8}\]

So far, we use the MM algorithm to obtain the numerical solution of the original minimization problem (2.3), i.e., Eq.(2.7). In addition, the sequence $\{ {u^k}\} $ generated by Eq.(2.7) clearly converges to the solution of the convex minimization problem (2.6). Since the original problem (2.3) is a nonconvex minimization problem, which may have multiple unequal solutions. Therefore, we will discuss the sequence $\{ {u^k}\} $ converges to a extremal solution of the original minimization problem (2.3) in the following.

\begin{lemma}
In \cite{b33}, the author has summarized the convergence of MM algorithm for solving various optimization problems.
Among them, in the case that the substitution function is constructed by the properties of the concave function, the solution sequence generated by the MM algorithm converges to a stationary point (also called stable point) of the original nonconvex minimization problem. Specifically, the sequence ${{u^k}}$ generated by Eq(2.7) converges to a stationary point of the minimization problem (2.3) as $k \to  + \infty $.
\end{lemma}

\section{The proposed model}

In this section, we first present an image exponential Retinex decomposition model based on weak space and hybrid nonconvex regularization. And then, this new model is numerically solved in the framework of ADMM combined with MM algorithm. Finally, we discuss the convergence conditions of proposed numerical algorithm.
	
\subsection{The proposed model}	

An observed image can be considered as a combination of multiple disjoint components. We can accurately model these different components by utilizing the Retinex theory, which decomposes an image into illumination and reflection. Building upon this concept, we propose the following image degradation model:
\[f = I \odot R + n=u+n,\tag{3.1}\]
where $f$ is a degraded image with noise, $u = I \odot R$ and $n$ represent the original clear image and gaussian additive noise, respectively.

In order to maintain the probability density of the Gaussian noise $n$, we apply a logarithmic transformation only to the original image $u$ as shown in Eq.(2.2). Additionally, we perform an exponential transformation on the components $i$ and $r$ in the logarithmic domain. By doing so, the model (3.1) can be expressed as follows:
\[f = {e^{i + r}} + n.\tag{3.2}\]

With the image modeling (3.2), we propose an exponential Retinex image decomposition model based on weak space and hybrid nonconvex regularization, which is defined as,

\[\mathop {\min }\limits_{i,r} \left\{ {E\left( {i,r} \right) = \frac{\lambda }{2}\left\| {f - {e^{i + r}}} \right\|_{{H^{ - 1}}}^2 + {\omega _1}\int_\Omega  {\varphi \left( {{\nabla ^2}i} \right)} dx + {\omega _2}\int_\Omega  {\varphi \left( {\nabla r} \right)} dx + \frac{\theta }{2}\left\| i \right\|_2^2} \right\},\tag{3.3}\]
where $\lambda$, ${{\omega _1}}$, ${{\omega _2}}$ and $\theta$ are  positive regularization parameters, and $\varphi \left(  \cdot  \right)$ is a nonconvex potential function. The term ${\left\| {f - {e^{i + r}}} \right\|_{{H^{ - 1}}}^2}$ is modeled based on the dual space ${H^{ - 1}}\left( \Omega  \right)$ of the Hilbert space ${H^1}\left( \Omega  \right)$, which aims to depict the oscillatory component. The two nonconvex regularization terms ${\int_\Omega  {\varphi \left( {{\nabla ^2}i} \right)} dx}$ and ${\int_\Omega  {\varphi \left( {\nabla r} \right)} dx}$ characterize the prior features of illumination $i$ and reflection $r$, respectively. The last term is to ensure the stability of this model, where the parameter $\theta $ is a very small value.

The model (3.3) for image exponential Retinex decomposition aims to characterize different components individually. By utilizing a weaker $H^{-1}$-norm, it effectively extracts the oscillation component from the degraded image $f$. The remaining components are further decomposed into a spatially smooth illumination $i$ and a piecewise constant reflection $r$, which are measured by two generalized nonconvex regularizers. This decomposition model allows for better extraction of oscillation information from an image, resulting in improved denoising effects. Figure 1 illustrates the block diagram of the proposed technique.

\begin{center}
	\begin{minipage}
		{1\textwidth}\centering {\includegraphics[width=10cm,height=6cm]{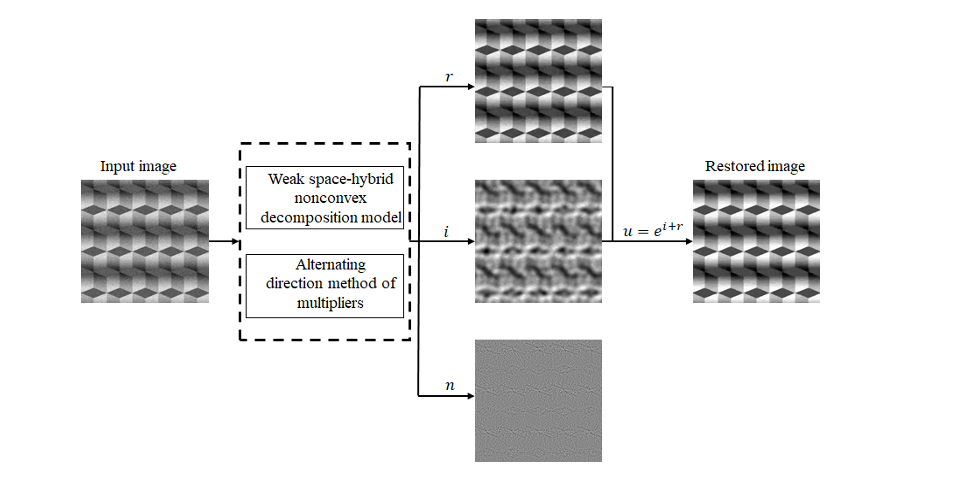}}
	\end{minipage}
	\\[1mm]
	\vskip2mm
	{\small \textbf{Fig.1} The block diagram of the proposed technique}
\end{center}

\begin{theorem}
If $u \in {H^{ - 1}}\left( \Omega  \right)$, then $F\left( u \right) = \left\| u \right\|_{{H^{ - 1}}}^2$ is a convex functional.
\end{theorem}

\begin{proof}
We need to first simplify the ${H^{ - 1}}$ norm in order to prove the convexity of functional $F\left( u \right)$. Osher et al. have provided a popular method to simplify the norm $\left\| u \right\|_{{H^{ - 1}}}^2$ in \cite{b26}. We assume $u = div\left( \textbf{g} \right)$ with $g \in {L^2}{\left( \Omega  \right)^2}$ and can use a unique Hodge decomposition of $\textbf{g}$ as
\[\textbf{g} = \nabla P + \textbf{Q},\]
where $P \in {H^1}\left( \Omega  \right)$ is a scalar function and $\textbf{Q}$ is a divergence-free vector field. And then, we can obtain that $u = div\left( \textbf{g} \right) = div\left( {\nabla P} \right) = \Delta P$, i.e., $P = {\Delta ^{ - 1}}u$. Finally, we ignore $\textbf{Q}$ from the expression $\textbf{g}$ and can obtain the following relationship as
\[F\left( u \right) = \left\| u \right\|_{{H^{ - 1}}}^2 = \left\| {\nabla \left( {{\Delta ^{ - 1}}u} \right)} \right\|_2^2.\tag{3.4}\]

We define $A: = \nabla {\Delta ^{ - 1}}$ to simplify the proof process. Therefore, we only need to prove that the functional $F\left( u \right) = \left\| {Au} \right\|_2^2$ satisfies the first-order condition of the convex functional, i.e.,
\[F\left( {{u_2}} \right) \ge F\left( {{u_1}} \right) + \nabla F\left( {{u_1}} \right)\left( {{u_2} - {u_1}} \right),\tag{3.5}\]
where ${u_1}, {u_2} \in \Omega $ are any two test functions. According to the fact that $\nabla F\left( {{u_1}} \right) = \nabla \left( {\left\| {A{u_1}} \right\|_2^2} \right) = 2{A^*}A{u_1}$ and $F\left( {{u_i}} \right) = \left\| {A{u_i}} \right\|_2^2 = \left\langle {A{u_i},\;A{u_i}} \right\rangle \left( {i = 1,2} \right)$, we can rewriter the formula (3.5) as
\[\begin{array}{l}
\left\langle {A{u_2},\;A{u_2}} \right\rangle  - \left\langle {A{u_1},\;A{u_1}} \right\rangle  - \left\langle {2{A^*}A{u_1},\;{u_2} - {u_1}} \right\rangle\\
= \left\langle {A{u_2},\;A{u_2}} \right\rangle  - \left\langle {A{u_1},\;A{u_1}} \right\rangle  - \left\langle {2A{u_1},\;A\left( {{u_2} - {u_1}} \right)} \right\rangle \\
= \left\langle {A{u_2},\;A{u_2}} \right\rangle  - 2\left\langle {A{u_1},\;A{u_2}} \right\rangle  + \left\langle {A{u_1},\;A{u_1}} \right\rangle \\
= \left\| {A{u_1},A{u_2}} \right\|_2^2.
\end{array}\tag{3.6}\]

The inequality $\left\| {A{u_1},A{u_2}} \right\|_2^2 \ge 0$ is clearly established, which implies the functional $F\left( u \right)$ satisfies condition (3.5) for any ${u_1}, {u_2} \in \Omega $. In summary, we prove that the functional $F\left( u \right)$ is strictly convex for $u$.
\end{proof}

\noindent\textbf{Remark.} In fact, Tang et al. \cite{b34} used a concrete example to verify that the ${H^{ - 1}}\left( \Omega  \right)$ space can accurately model oscillation information. In this example, a one-dimensional oscillation function $\sin 2\pi nx \; (x \in \left[ {0,1} \right]) $ is used to validate this point. This example indicates that $\left\| {\sin 2\pi nx} \right\|_{{H^{ - 1}}}^2 \to 0$ as $n \to \infty $. Through this example, we can know that a periodic function has a smaller norm in ${H^{ - 1}}\left( \Omega  \right)$ space.  Therefore, it is appropriate to employ   ${H^{ - 1}}\left( \Omega  \right)$ space to measure the oscillation components in the minimization model since their ${H^{ - 1}}$-norm is very small.

\subsection{Numerical algorithm}
	
In this subsection, a numerical algorithm combined with ADMM is proposed to solve the non-convex optimization problem. Next, we present the specific solution process of porposed model (3.3) in the following.

We introduce and constrain the auxiliary variables $v = i + r$, $u = {e^v}$, $\textbf{m} = {\nabla ^2}i$ and $\textbf{n} = \nabla r$ in order to simplify the complex coupling relationship between variables $i$ and $r$ in the optimization problem (3.3). At this point, the augmented Lagrangian function $L$ of the original minimization problem can be rewritten as
\[\begin{array}{*{20}{l}}
{L\left( {i,r,v,u,{\bf{m}},{\bf{n}};{y_1},{{\bf{y}}_{\bf{2}}},{{\bf{y}}_{\bf{3}}}} \right) = \frac{\lambda }{2}\left\| {f - u} \right\|_{{H^{ - 1}}}^2 + {\omega _1}\int_\Omega  {\varphi \left( {\bf{m}} \right)} dx}\\
{\;\; + {\omega _2}\int_\Omega  {\varphi \left( {\bf{n}} \right)} dx + \frac{\theta }{2}\left\| i \right\|_2^2 + \frac{\beta }{2}\left\| {v - i - r} \right\|_2^2 + \frac{\rho }{2}\left\| {u - {e^v} + \frac{{{y_1}}}{\rho }} \right\|_2^2}\\
{\;\; + \frac{\rho }{2}\left\| {{\bf{m}} - {\nabla ^2}i + \frac{{{{\bf{y}}_{\bf{2}}}}}{\rho }} \right\|_2^2 + \frac{\rho }{2}\left\| {{\bf{n}} - \nabla r + \frac{{{{\bf{y}}_{\bf{3}}}}}{\rho }} \right\|_2^2 - C},
\end{array}\tag{3.7}\]
where ${{y_1}}$, ${{{\bf{y}}_{\bf{2}}}}$ and ${{{\bf{y}}_{\bf{3}}}}$ are the Lagrangian multipliers, and $\beta$, $\rho  > 0$ are the parameters of penalty terms. It should be pointed out that $C{\rm{ = }}\frac{1}{{2\rho }}\left( {\left\| {{y_1}} \right\|_2^2 + \left\| {{{\bf{y}}_{\bf{2}}}} \right\|_2^2 + \left\| {{{\bf{y}}_3}} \right\|_2^2} \right)$ can be omitted in the minimization process. Then, we set the initial value of the variable and use the alternating direction iterative algorithm to optimize the augmented Lagrangian function $L$. In this way, we convert the solution of the original minimization problem into the optimization of several subproblems.

In fact, the difficulty of solving the subproblems is much lower compared to directly solving the original minimization model (3.3). Moreover, due to the non-convex or nonlinear nature of these  subproblems, we can only resort to the MM algorithm to find numerical solutions instead of analytical solutions. In the following, we provide a detailed description of the solving process for all subproblems.

\textbf{(1) $v$-subproblem}

Firstly, the $v$-subproblem can be described as
\[{v^{k + 1}} = \mathop {\arg \min }\limits_v \frac{\beta }{2}\left\| {v - {i^k} - {r^k}} \right\|_2^2 + \frac{\rho }{2}\left\| {{u^k} - {e^v} + \frac{{y_1^k}}{\rho }} \right\|_2^2.\tag{3.8}\]

It should be noted that ${{e^v}}$ is a nonlinear functional, and it is difficult to obtain the analytic solution of the variable $v$ directly by using the necessary condition of functional extremum. We perform a first-order linear Taylor expansion of the nonlinear term $\frac{\rho }{2}||{u^k} - {e^{{v^k}}} + \frac{{y_1^k}}{\rho }||_2^2$ at point ${v^k}$ as \cite{b35}
\[\frac{\rho }{2}\left\| {{u^k} - {e^v} + \frac{{y_1^k}}{\rho }} \right\|_2^2 = \frac{\rho }{2}\left\| {{e^{{v^k}}} - {u^k} - \frac{{y_1^k}}{\rho }} \right\|_2^2 + \rho {e^{{v^k}}}\left( {{e^{{v^k}}} - {u^k} - \frac{{y_1^k}}{\rho }} \right)\left( {v - {v^k}} \right) + R\left( v \right),\]
where $R\left( v \right)$ is the remainder of a Taylor expansion that can be approximated by a quadratic term. Therefore, we rewrite the minimization problem (3.8) to
\[\begin{array}{l}
{v^{k + 1}} = \mathop {\arg \min }\limits_v \frac{\beta }{2}\left\| {v - {i^k} - {r^k}} \right\|_2^2 + \frac{\rho }{2}\left\| {{e^{{v^k}}} - {u^k} - \frac{{y_1^k}}{\rho }} \right\|_2^2\\
\;\;\;\;\;\;\;\;\;\; + \rho {e^{{v^k}}}\left( {{e^{{v^k}}} - {u^k} - \frac{{y_1^k}}{\rho }} \right)\left( {v - {v^k}} \right) + \frac{\tau }{2}\left\| {v - {v^k}} \right\|_2^2,
\end{array}\tag{3.9}\]
where $\left| \tau  \right|$ is a slight number, and the last term is to approximate the Taylor remainder. According to the Euler-Lagrange equation, we can easily find the necessary conditions for the solution of the above optimization problem (3.9) as
\[\beta \left( {v - {i^k} - {r^k}} \right) + \rho {e^{{v^k}}}\left( {{e^{{v^k}}} - {u^k} - \frac{{y_1^k}}{\rho }} \right) + \tau \left( {v - {v^k}} \right) = 0.\]
Furthermore, we can obtain the analytical solution of the $v$-subproblem as
\[{v^{k + 1}} = \frac{{\beta \left( {{i^k} + {r^k}} \right) + \tau {v^k} - \rho {e^{{v^k}}}\left( {{e^{{v^k}}} - {u^k}} \right) + y_1^k{e^{{v^k}}}}}{{\beta  + \tau }}.\tag{3.10}\]

\textbf{(2) $u$-subproblem}

According to Theorem 3.1, we transform the ${H^{ - 1}}$-norm into the ordinary ${L^2}$-norm. The $u$-subproblem is a convex smooth minimization problem, which is described as
\[{u^{k + 1}} = \mathop {\arg \min }\limits_u \frac{\lambda }{2}\left\| {\nabla \left( {{\Delta ^{ - 1}}\left( {u - f} \right)} \right)} \right\|_2^2 + \frac{\rho }{2}\left\| {u - {e^{{v^{k + 1}}}} + \frac{{y_1^k}}{\rho }} \right\|_2^2,\tag{3.11}\]
where $\nabla $ and ${{\Delta ^{ - 1}}}$ are the gradient operator and the Laplace inverse operator, respectively. The solution of this minimization problem satisfies Euler-Lagrange equation, i.e.,
\[\lambda {\left( {\nabla {\Delta ^{ - 1}}} \right)^*}\left( {\nabla {\Delta ^{ - 1}}\left( {u - f} \right)} \right) + \rho \left( {u - {e^{{v^{k + 1}}}} + \frac{{y_1^k}}{\rho }} \right) = 0,\tag{3.12}\]
where operator ${\nabla ^*}$ is the conjugate operator of operator $\nabla $, and they also satisfy the relation ${\nabla ^*}\nabla  =  - div\left( \nabla  \right) =  - \Delta $. Both sides of Eq.(3.12) can be rewritten by the $\Delta $ operator as
\[\left( {\lambda I - \rho \Delta } \right)u = \lambda f + \rho \Delta \left( { - {e^{{v^{k + 1}}}} + \frac{{y_1^k}}{\rho }} \right),\tag{3.13}\]
where $I$ is an identity matrix of the same size as $u$. It should be noted that the $\Delta $ operator in Eq.(3.13) is a convolution operation, so we cannot directly separate the variable $u$ from it. A general approach is to use Fourier transform to transform the convolution operation in the spatial domain into the ordinary multiplication in the frequency domain, thus obtaining an analytical solution of the $u$-subproblem as
\[u^{k+1} = {F^{ - 1}}\left\{ {\frac{{F\left[ {\lambda f + \rho \Delta \left( { - {e^{{v^{k + 1}}}} + \frac{{y_1^k}}{\rho }} \right)} \right]}}{{F\left( {\lambda I - \rho \Delta } \right)}}} \right\},\tag{3.14}\]
here and later, $F\left(  \cdot  \right)$ and ${F^{ - 1}}\left(  \cdot  \right)$ are defined as Fourier transform and inverse Fourier transform operations, respectively.

\textbf{(3) $i$ and $r$-subproblem}

Firstly, the $i$-subproblem is described as
\[{i^{k + 1}} = \mathop {\arg \min }\limits_i \frac{\beta }{2}\left\| {{v^{k + 1}} - i - {r^k}} \right\|_2^2 + \frac{\rho }{2}\left\| {{{\bf{m}}^k} - {\nabla ^2}i + \frac{{{\bf{y}}_2^k}}{\rho }} \right\|_2^2 + \frac{\theta }{2}\left\| i \right\|_2^2.\tag{3.15}\]

Obviously, this minimization problem is easy to be solved because it is convex and smooth. By using the variational method, we can directly obtain the necessary conditions for the solution of the minimization problem (3.15) as
\[\left( {\beta I + \theta I + \rho {{\left( {{\nabla ^2}} \right)}^*}{\nabla ^2}} \right)i = \beta \left( {{v^{k + 1}} - {r^k}} \right) + \rho {\left( {{\nabla ^2}} \right)^*}\left( {{{\bf{m}}^k} + \frac{{{\bf{y}}_2^k}}{\rho }} \right),\tag{3.16}\]
where ${\left( {{\nabla ^2}} \right)^*}$ is a conjugate operator of ${{\nabla ^2}}$. Similarly, Eq.(3.16) also contains two-dimensional convolution operation. Assuming that $i$ satisfies the periodic boundary condition, it is not difficult to obtain ${i^{k + 1}}$ from the above linear equation by Fourier transform and its inverse transform as
\[{i^{k + 1}} = {F^{{\rm{ - 1}}}}\left\{ {\frac{{F\left[ {\beta \left( {{v^{k + 1}} - {r^k}} \right) + \rho {{\left( {{\nabla ^2}} \right)}^*}\left( {{{\bf{m}}^k} + \frac{{{\bf{y}}_2^k}}{\rho }} \right)} \right]}}{{F\left( {\beta I + \theta I + \rho {{\left( {{\nabla ^2}} \right)}^*}{\nabla ^2}} \right)}}} \right\}.\tag{3.17}\]

Next, we use similar steps to solve the $r$-subproblem
\[{r^{k + 1}} = \mathop {\arg \min }\limits_r \frac{\beta }{2}\left\| {{v^{k + 1}} - {i^{k + 1}} - r} \right\|_2^2 + \frac{\rho }{2}\left\| {{{\bf{n}}^k} - \nabla r + \frac{{{\bf{y}}_3^k}}{\rho }} \right\|_2^2.\tag{3.18}\]

Similarly, according to the Euler-Lagrangian equation, the necessary condition for the solution of the minimization problem (3.18) is
\[\left( {\beta I + \rho {\nabla ^*}\nabla } \right)r = \beta \left( {{v^{k + 1}} - {i^{k + 1}}} \right) + \rho {\nabla ^*}\left( {{{\bf{n}}^k} + \frac{{{\bf{y}}_3^k}}{\rho }} \right).\tag{3.19}\]
Further, we solve the iterative form of ${r^{k + 1}}$ by the Fourier domain, which is described as
\[{r^{k + 1}} = {F^{ - 1}}\left\{ {\frac{{F\left[ {\beta \left( {{v^{k + 1}} - {i^{k + 1}}} \right) + \rho {\nabla ^*}\left( {{{\bf{n}}^k} + \frac{{{\bf{y}}_3^k}}{\rho }} \right)} \right]}}{{F\left( {\beta I + \rho {\nabla ^*}\nabla } \right)}}} \right\}.\tag{3.20}\]

\textbf{(4) $\mathbf{m}$ and $\mathbf{n}$-subproblem}

Now, we solve ${\bf{m}}$ and ${\bf{n}}$ nonconvex subproblems. In fact, we only need to solve one of the minimization problems, because the forms of these two non-convex optimization problems are similar. Here we solve the ${\bf{m}}$-subproblem in detail, which is written as
\[{{\bf{m}}^{k + 1}} = \mathop {\arg \min }\limits_{\bf{m}} {\omega _1}\int_\Omega  {\varphi \left( {\bf{m}} \right)} dx + \frac{\rho }{2}\left\| {{\bf{m}} - {\nabla ^2}{i^{k + 1}} + \frac{{{\bf{y}}_2^k}}{\rho }} \right\|_2^2.\tag{3.21}\]
There is not difficult to perceive that the minimization problem (3.21) is a classical nonconvex optimization problem. In Section 2.2, we introduce in detail the solution of such problems by MM algorithm, whcih is updated in the form of Eq.(2.7). Therefore, we directly obtain the solution of problem (3.21) as
\[{{\bf{m}}^{k + 1}} = {\nabla ^2}{i^{k + 1}} - \frac{{{\bf{y}}_2^k}}{\rho } - \frac{{{\omega _1}\nabla \varphi \left( {{{\bf{m}}^k}} \right)}}{\rho }.\tag{3.22}\]

Next, we solve the $\bf{n}$-subproblem similarly,
\[{{\bf{n}}^{k + 1}} = \mathop {\arg \min }\limits_{\bf{n}} {\omega _2}\int_\Omega  {\varphi \left( {\bf{n}} \right)} dx + \frac{\rho }{2}\left\| {{\bf{n}} - \nabla {r^{k + 1}} + \frac{{{\bf{y}}_3^{k + 1}}}{\rho }} \right\|_2^2.\tag{3.23}\]
Likewise, we obtain the iterative form of variable ${{\bf{n}}^{k + 1}}$ in the framework of MM algorithm as
\[{{\bf{n}}^{k + 1}} = \nabla {r^{k + 1}} - \frac{{{\bf{y}}_3^k}}{\rho } - \frac{{{\omega _2}\nabla \varphi \left( {{{\bf{n}}^k}} \right)}}{\rho }.\tag{3.24}\]

\textbf{(5) Updating Lagrange multipliers}

Finally, the Lagrange multipliers ${{y_1}}$, ${{{\bf{y}}_{\bf{2}}}}$ and ${{{\bf{y}}_{\bf{3}}}}$ are updated according to the ADMM algorithm as follows
\[\left\{ \begin{array}{l}
y_1^{k + 1} = y_1^k + \rho \left( {{u^{k + 1}} - {e^{{v^{k + 1}}}}} \right)\\
{\bf{y}}_2^{k + 1} = {\bf{y}}_2^k + \rho \left( {{{\bf{m}}^{k + 1}} - {\nabla ^2}{i^{k + 1}}} \right)\\
{\bf{y}}_3^{k + 1} = {\bf{y}}_3^k + \rho \left( {{{\bf{n}}^{k + 1}} - \nabla {r^{k + 1}}} \right)
\end{array} \right..\tag{3.25}\]

To tackle the challenges posed by the proposed model (3.3), we have successfully combined multiple numerical methods within the framework of ADMM. This integration allows us to achieve efficient and accurate solutions for image decomposition. In Algorithm 1, we present a step-by-step procedure that outlines our algorithm in detail.\\

\noindent\textbf{Remark.} Before starting the iteration process, several variables and parameters need to be initialized in Algorithm 1.  These include ${i^0}$, ${r^0}$, ${v^0}$, ${u^0}$, $\bf{m}^0$, ${{\bf{n}}^0}$ $y_1^0$, ${\bf{y}}_2^0$, and ${\bf{y}}_3^0$. Additionally, parameters such as $\lambda$, ${{\omega _1}}$, ${{\omega _2}}$, $\theta$, $\beta$, and the penalty parameter $\rho > 0$ must be set initially. Once these initializations are complete, the algorithm proceeds to iterate starting from $k = 0$. The iteration continues until the relative error of the variable $u^{k +1}$ is greater than a preset accuracy threshold denoted as $tol$ (which is defined in Section 4). The algorithm terminates when the relative error falls below this threshold.
It is worth noting that Algorithm 1 consists of only one loop. This indicates that its complexity is determined by the size of the input image ($M$) and the number of iterations ($N$) performed by the algorithm. Therefore, we can conclude that the complexity of Algorithm 1 is on the order of $O(M \times N)$.

\begin{flushleft}
		\begin{tabular}{l}
			\hline
			\textbf{Algorithm 1. For solving the proposed model (2.3)}\\
			\hline
			\textbf{Initialize:} ${i^0}$, ${r^0}$, ${v^0}$, ${u^0}$, ${{\bf{m}}^0}$, ${{\bf{n}}^0}$, $y_1^0$, ${\bf{y}}_2^0$ and ${\bf{y}}_3^0$;\\
			 \textbf{Set parameters:} $\lambda $, ${{\omega _1}}$, ${{\omega _2}}$, $\theta $, $\beta $, $\rho  > 0$ and $tol$;
			\\ \textbf{While} relative-error $>tol$, \textbf{do}
			\\ \qquad Compute $v^{k+1}$ by Eq.(3.10);
			\\ \qquad Compute $u^{k+1}$ by Eq.(3.14);
			\\ \qquad Compute $i^{k+1}$ by Eq.(3.17);
			\\ \qquad Compute $r^{k+1}$ by Eq.(3.20);
			\\ \qquad Compute ${{\bf{m}}^{k + 1}}$ by Eq.(3.22);
			\\ \qquad Compute ${{\bf{n}}^{k + 1}}$ by Eq.(3.24);
			\\ \qquad Updata Lagrange multipliers $y_1^{k + 1}$, ${\bf{y}}_2^{k + 1}$ and ${\bf{y}}_3^{k + 1}$ by Eqs.(3.25);
			\\ \qquad Set $k=k+1$;
			\\ \textbf{End while}
			\\ \textbf{Output:} $u^{k+1}$, $i^{k+1}$ and $r^{k+1}$.
			\\ \hline
		\end{tabular}
\end{flushleft}
\subsection{Convergence analysis}

In this subsection, we will discuss the convergence of the proposed algorithm, which is solved using the ADMM framework. The algorithm involves solving two non-convex minimization subproblems, namely ${\bf{m}}$ and ${\bf{n}}$, using the MM algorithm. It is important to note that while the convergence analysis of ADMM for convex optimization problems is well-established \cite{b27,b28}, studying the convergence of non-convex optimization problems presents a significant challenge

Now, we need to make several crucial assumptions that are essential for proving the convergence of the algorithm.

\begin{assumption}
	The non-convex potential function $\varphi \left(  \cdot  \right)$ is closed, proper and lower semi-continuity. In addition, its gradient function $\nabla \varphi \left(  \cdot  \right)$ satisfies the Lipschitz continuity condition, i.e., there exists a positive constant $K$ such that
	\[{\left\| {\nabla \varphi \left( x \right) - \nabla \varphi \left( y \right)} \right\|_2} \le K{\left\| {x - y} \right\|_2},\;\forall x,\;y \in dom\left( \varphi  \right).\]
\end{assumption}

\begin{assumption}
	The penalty coefficients $\rho$ and $\beta$ are large enough to satisfy $\rho>K$ and $\beta>\left| \tau  \right|$. In this case, the equivalent extreme value problem (3.9) of $v$ subproblem, $\bf{m}$ subproblem (3.21) and $\bf{n}$ subproblem (3.23) are strictly convex.
\end{assumption}

\begin{assumption}
	The energy functional $E\left( {i,r} \right)$ is bounded below, i.e., $\underline E  = \min E\left( {i,r} \right) >  - \infty $.
\end{assumption}

In the following work, we prove the convergence of Algorithm 1 with three panels in turn. Firstly, Theorem 3.5-3.8 prove that the augmented Lagrangian function $L$ decreases monotonically as $k \to  + \infty$. Secondly, Theorem 3.9 prove that the augmented Lagrangian function $L$ is bounded below and convergent. Finally, Theorems 3.10-3.11 prove that the sequence ${z^k} = \left( {{i^k},{r^k},{v^k},{u^k},{{\bf{m}}^k},{{\bf{n}}^k};y_1^k,{\bf{y}}_2^k,{\bf{y}}_3^k} \right)$ generated by Algorithm 1 converges to a limit point ${z^*} = \left( {{i^*},{r^*},{v^*},{u^*},{{\bf{m}}^*},{{\bf{n}}^*};y_1^*,{\bf{y}}_2^*,{\bf{y}}_3^*} \right)$ as $k \to  + \infty$, where $z^*$ is a stationary point of the augmented Lagrangian function $L$.

\begin{theorem}
	Let the sequence $ \left( {{i^k},{r^k},{v^k},{u^k},{{\bf{m}}^k},{{\bf{n}}^k};y_1^k,{\bf{y}}_2^k,{\bf{y}}_3^k} \right)$ generated by Algorithm 1, and we can infer that there must be three positive constants $\gamma _1$, $\gamma _2$, and  $\gamma _3$ such that
	\begin{numcases}{}
	\begin{array}{l}
	L\left( {{i^{k + 1}},{r^k},{v^k},{u^{k}},{{\bf{m}}^k},{{\bf{n}}^k};y_1^k,{\bf{y}}_2^k,{\bf{y}}_3^k} \right) - L\left( {{i^k},{r^k},{v^k},{u^{k}},{{\bf{m}}^k},{{\bf{n}}^k};y_1^k,{\bf{y}}_2^k,{\bf{y}}_3^k} \right)\\
	\;\;\; \le  - \frac{{{\gamma _1}}}{2}\left\| {{i^{k + 1}} - {i^k}} \right\|_2^2
	\end{array} \tag{3.26a}\\
	\begin{array}{l}
	L\left( {{i^{k + 1}},{r^{k + 1}},{v^k},{u^{k}},{{\bf{m}}^k},{{\bf{n}}^k};y_1^k,{\bf{y}}_2^k,{\bf{y}}_3^k} \right) - L\left( {{i^{k + 1}},{r^k},{v^k},{u^{k}},{{\bf{m}}^k},{{\bf{n}}^k};y_1^k,{\bf{y}}_2^k,{\bf{y}}_3^k} \right)\\
	\;\;\; \le  - \frac{{{\gamma _2}}}{2}\left\| {{r^{k + 1}} - {r^k}} \right\|_2^2
	\end{array} \tag{3.26b}\\
	\begin{array}{l}
	L\left( {{i^{k + 1}},{r^{k + 1}},{v^{k + 1}},{u^{k }},{{\bf{m}}^k},{{\bf{n}}^k};y_1^k,{\bf{y}}_2^k,{\bf{y}}_3^k} \right) - L\left( {{i^{k + 1}},{r^{k + 1}},{v^k},{u^{k}},{{\bf{m}}^k},{{\bf{n}}^k};y_1^k,{\bf{y}}_2^k,{\bf{y}}_3^k} \right)\\
	\;\;\; \le  - \frac{{{\gamma _3}}}{2}\left\| {{v^{k + 1}} - {v^k}} \right\|_2^2
	\end{array} \tag{3.26c}
	\end{numcases}
\end{theorem}

\begin{proof}
	We first prove that the augmented Lagrangian function $L$ is strictly convex with respect to $i$. Obviously, the terms ${\frac{\theta }{2}\left\| i \right\|_2^2}$ and ${\frac{\beta }{2}\left\| {v - i - r} \right\|_2^2}$ in the function $L$ are quadratic smooth terms with respect to $i$, so they are strictly convex with respect to $i$. The seventh term in the function $L$ can be expanded into
	\[\left\| {{\bf{m}} - {\nabla ^2}i + \frac{{{{\bf{y}}_{\bf{2}}}}}{\rho }} \right\|_2^2{\rm{ = }}\left\| {{\nabla ^2}i} \right\|_{\rm{2}}^{\rm{2}}{\rm{ + }}\left\langle {{\nabla ^2}i, - {\bf{m}} - \frac{{{{\bf{y}}_{\bf{2}}}}}{\rho }} \right\rangle  + \left\| { - {\bf{m}} - \frac{{{{\bf{y}}_{\bf{2}}}}}{\rho }} \right\|_2^2,\tag{3.27}\]
	and the second term in Eq(3.27) \[\left\langle {{\nabla ^2}i, - {\bf{m}} - \frac{{{{\bf{y}}_{\bf{2}}}}}{\rho }} \right\rangle  = \left\langle {i,{{\left( {{\nabla ^2}} \right)}^*}\left( { - {\bf{m}} - \frac{{{{\bf{y}}_{\bf{2}}}}}{\rho }} \right)} \right\rangle\] is a linear term with respect to $i$, so it is convex. In addition, we can infer that the term $\left\| {{\nabla ^2}i} \right\|$ in Eq(3.27) is convex with respect to $i$, which can be obtained by Eq(3.6).
	
	 By the above discussion, we conclude that the function ${L\left( {i,r,v,u,{\bf{m}},{\bf{n}};{y_1},{{\bf{y}}_{\bf{2}}},{{\bf{y}}_{\bf{3}}}} \right)}$ is strictly convex with respect to $i$. According to the definition of convex function, we can infer that there must be a positive constant $\gamma_1$ such that
	\[\begin{array}{*{20}{l}}
	{\begin{array}{*{20}{l}}
		{L\left( {{i^k},{r^k},{v^k},{u^k},{{\bf{m}}^k},{{\bf{n}}^k};y_1^k,{\bf{y}}_2^k,{\bf{y}}_3^k} \right)}\\
		{\;\;\; \ge \frac{{{\gamma _1}}}{2}\left\| {{i^k} - {i^{k + 1}}} \right\|_2^2 + L\left( {{i^{k + 1}},{r^k},{v^k},{u^k},{{\bf{m}}^k},{{\bf{n}}^k};y_1^k,{\bf{y}}_2^k,{\bf{y}}_3^k} \right)}
		\end{array}}\\
	{\;\;\;\;\;\; - \left\langle {{\nabla _i} \cdot L\left( {{i^{k + 1}},{r^k},{v^k},{u^k},{{\bf{m}}^k},{{\bf{n}}^k};y_1^k,{\bf{y}}_2^k,{\bf{y}}_3^k} \right),{i^k} - {i^{k + 1}}} \right\rangle .}
	\end{array}\tag{3.28}\]
	In the $i$ subproblem, since $i^{k+1}$ is a minimum point of the function $L\left( {i,{r^k},{v^k},{u^k},{{\bf{m}}^k},{{\bf{n}}^k};y_1^k,{\bf{y}}_2^k,{\bf{y}}_3^k} \right)$, according to the necessary condition of the extreme point, we have
	\[{\nabla _i} \cdot L\left( {{i^{k + 1}},{r^k},{v^k},{u^k},{{\bf{m}}^k},{{\bf{n}}^k};y_1^k,{\bf{y}}_2^k,{\bf{y}}_3^k} \right) = 0.\tag{3.29}\]
	And combining (3.28) and (3.29), we conclude that
	\[\begin{array}{*{20}{l}}
		{L\left( {{i^{k + 1}},{r^k},{v^k},{u^k},{{\bf{m}}^k},{{\bf{n}}^k};y_1^k,{\bf{y}}_2^k,{\bf{y}}_3^k} \right)}\\
		{\;\; - L\left( {{i^k},{r^k},{v^k},{u^k},{{\bf{m}}^k},{{\bf{n}}^k};y_1^k,{\bf{y}}_2^k,{\bf{y}}_3^k} \right) \le  - \frac{{{\gamma _1}}}{2}\left\| {{i^{k + 1}} - {i^k}} \right\|_2^2.}
	\end{array}\]
	
	Similarly, we can also infer that there must be a positive constant $\gamma_2$ such that
	\[\begin{array}{*{20}{l}}
	{L\left( {{i^{k + 1}},{r^{k + 1}},{v^k},{u^k},{{\bf{m}}^k},{{\bf{n}}^k};y_1^k,{\bf{y}}_2^k,{\bf{y}}_3^k} \right)}\\
	{\;\; - L\left( {{i^{k + 1}},{r^k},{v^k},{u^k},{{\bf{m}}^k},{{\bf{n}}^k};y_1^k,{\bf{y}}_2^k,{\bf{y}}_3^k} \right) \le  - \frac{{{\gamma _2}}}{2}\left\| {{r^{k + 1}} - {r^k}} \right\|_2^2.}
	\end{array}\]
	
	Finally, we prove that inequality (3.26c) holds. It is worth noting that we can not directly determine the convexity of the nonlinear term ${\frac{\rho }{2}\left\| {u - {e^v} + \frac{{{y_1}}}{\rho }} \right\|_2^2}$ with respect to $v$. In fact, we use the first-order linear Taylor expansion to solve the $v$ subproblem well. Therefore, according to the equivalent extreme value problem (3.9) of the $v$ subproblem (3.8), we only need to judge the convexity of the following formula with respect to $v$,
	\[ \frac{\beta }{2}\left\| {v - i - r} \right\|_2^2 + \frac{\rho }{2}\left\| {{e^{{v^k}}} - u - \frac{{{y_1}}}{\rho }} \right\|_2^2\; + \rho {e^{{v^k}}}\left( {{e^{{v^k}}} - u - \frac{{{y_1}}}{\rho }} \right)\left( {v - {v^k}} \right) + \frac{\tau }{2}\left\| {v - {v^k}} \right\|_2^2,\]
	where $v^{k}$ is equivalent to a constant. Obviously, the first and third terms in the above formula are the quadratic term and the first term about $v$ respectively, so they are both convex with respect to $v$. The parameter $\tau$ in the Taylor remainder $\frac{\tau }{2}\left\| {v - {v^k}} \right\|_2^2$ cannot determine the sign, because we cannot determine the convexity of the nonlinear term ${\frac{\rho }{2}\left\| {u - {e^v} + \frac{{{y_1}}}{\rho }} \right\|_2^2}$ with respect to $v$. Fortunately, we can deduce that the expression $\frac{\beta }{2}\left\| {v - i - r} \right\|_2^2 + \frac{\tau }{2}\left\| {v - {v^k}} \right\|_2^2$ is strictly convex with respect to $v$, which is attributed to the fact that $\beta  > \left| \tau  \right|$ is satisfied in Assumption 3.3.
	
	By the above discussion, we conclude that the function ${L\left( {i,r,v,u,{\bf{m}},{\bf{n}};{y_1},{{\bf{y}}_{\bf{2}}},{{\bf{y}}_{\bf{3}}}} \right)}$ is strictly convex with respect to $v$. Similarly, we can infer that there must be a positive constant $\gamma_3$ such that
	\[\begin{array}{*{20}{l}}
	{L\left( {{i^{k + 1}},{r^{k + 1}},{v^{k + 1}},{u^k},{{\bf{m}}^k},{{\bf{n}}^k};y_1^k,{\bf{y}}_2^k,{\bf{y}}_3^k} \right)}\\
	{\;\; - L\left( {{i^{k + 1}},{r^{k + 1}},{v^k},{u^k},{{\bf{m}}^k},{{\bf{n}}^k};y_1^k,{\bf{y}}_2^k,{\bf{y}}_3^k} \right) \le  - \frac{{{\gamma _3}}}{2}\left\| {{v^{k + 1}} - {v^k}} \right\|_2^2.}
	\end{array}\]
	
	The desired result is obtained.\qed
\end{proof}

\begin{theorem}
	Let the sequence $ \left( {{i^k},{r^k},{v^k},{u^k},{{\bf{m}}^k},{{\bf{n}}^k};y_1^k,{\bf{y}}_2^k,{\bf{y}}_3^k} \right)$ generated by Algorithm 1, and we can infer that there must be three positive constants $\eta _1$, $\eta_2$ and $\eta _3$ such that
	\begin{numcases}{}
	\begin{array}{l}
	L\left( {{i^{k + 1}},{r^{k + 1}},{v^{k + 1}},{u^{k + 1}},{{\bf{m}}^k},{{\bf{n}}^k};y_1^k,{\bf{y}}_2^k,{\bf{y}}_3^k} \right)\\
	\;\;\; - L\left( {{i^{k + 1}},{r^{k + 1}},{v^{k + 1}},{u^k},{{\bf{m}}^k},{{\bf{n}}^k};y_1^k,{\bf{y}}_2^k,{\bf{y}}_3^k} \right) \le  - \frac{{{\eta _1}}}{2}\left\| {{u^{k + 1}} - {u^k}} \right\|_2^2
	\end{array} \tag{3.30a}\\
	\begin{array}{*{20}{l}}
	{L\left( {{i^{k + 1}},{r^{k + 1}},{v^{k + 1}},{u^{k + 1}},{{\bf{m}}^{k + 1}},{{\bf{n}}^k};y_1^k,{\bf{y}}_2^k,{\bf{y}}_3^k} \right)}\\
	{\;\;\; - L\left( {{i^{k + 1}},{r^{k + 1}},{v^{k + 1}},{u^{k + 1}},{{\bf{m}}^k},{{\bf{n}}^k};y_1^k,{\bf{y}}_2^k,{\bf{y}}_3^k} \right) \le  - \frac{{{\eta _2}{\omega _1}}}{2}\left\| {{{\bf{m}}^{k + 1}}{\rm{ - }}{{\bf{m}}^k}} \right\|_2^2}
	\end{array} \tag{3.30b}\\
	\begin{array}{*{20}{l}}
	{L\left( {{i^{k + 1}},{r^{k + 1}},{v^{k + 1}},{u^{k + 1}},{{\bf{m}}^{k + 1}},{{\bf{n}}^{k + 1}};y_1^k,{\bf{y}}_2^k,{\bf{y}}_3^k} \right)}\\
	{\;\;\; - L\left( {{i^{k + 1}},{r^{k + 1}},{v^{k + 1}},{u^{k + 1}},{{\bf{m}}^{k + 1}},{{\bf{n}}^k};y_1^k,{\bf{y}}_2^k,{\bf{y}}_3^k} \right) \le  - \frac{{{\eta _3}{\omega _2}}}{2}\left\| {{{\bf{n}}^{k + 1}}{\rm{ - }}{{\bf{n}}^k}} \right\|_2^2}
	\end{array} \tag{3.30c}
	\end{numcases}
\end{theorem}

\begin{proof}
	According to Theorem 3.1, the function ${L\left( {i,r,v,u,{\bf{m}},{\bf{n}};{y_1},{{\bf{y}}_{\bf{2}}},{{\bf{y}}_{\bf{3}}}} \right)}$ is strictly convex with respect to $u$. Therefore, we can infer that there must be a positive constant $\eta_1$ such that
	\[ \begin{array}{*{20}{l}}
	{L\left( {{i^{k + 1}},{r^{k + 1}},{v^{k + 1}},{u^{k + 1}},{{\bf{m}}^k},{{\bf{n}}^k};y_1^k,{\bf{y}}_2^k,{\bf{y}}_3^k} \right)}\\
	{\;\; - L\left( {{i^{k + 1}},{r^{k + 1}},{v^{k + 1}},{u^k},{{\bf{m}}^k},{{\bf{n}}^k};y_1^k,{\bf{y}}_2^k,{\bf{y}}_3^k} \right) \le  - \frac{{{\eta _1}}}{2}\left\| {{u^{k + 1}} - {u^k}} \right\|_2^2.}
	\end{array}\]
	
	Next, we prove that the inequality (3.30b) holds. There is a fact $F\left( {{{\bf{m}}^{k + 1}}} \right) \le F\left( {{{\bf{m}}^k}} \right)$, because ${{{\bf{m}}^{k + 1}}}$ is a minimum point of the minimization problem (3.21). We have
	\[\begin{array}{*{20}{l}}
	{{\omega _1}\left\langle {\nabla \varphi \left( {{{\bf{m}}^k}} \right),{{\bf{m}}^{k + 1}} - {{\bf{m}}^k}} \right\rangle  + \frac{\rho }{2}\left\| {{{\bf{m}}^{k + 1}} - {\nabla ^2}{i^{k + 1}} + \frac{{{\bf{y}}_2^k}}{\rho }} \right\|_2^2}\\
	{\;\;\; \le \frac{\rho }{2}\left\| {{{\bf{m}}^k} - {\nabla ^2}{i^{k + 1}} + \frac{{{\bf{y}}_2^k}}{\rho }} \right\|_2^2.}
	\end{array}\tag{3.31}\]
	In addition, since the potential function ${\varphi \left( {{{\bf{m}}}} \right)}$ is nonconvex, there must be a positive constant $\eta_2$ such that
	\[ \varphi \left( {{{\bf{m}}^{k + 1}}} \right) \le \varphi \left( {{{\bf{m}}^k}} \right) + \left\langle {\nabla \varphi \left( {{{\bf{m}}^k}} \right),{{\bf{m}}^{k + 1}} - {{\bf{m}}^k}} \right\rangle  - \frac{{{\eta _2}}}{2}\left\| {{{\bf{m}}^{k + 1}} - {{\bf{m}}^k}} \right\|_2^2.\tag{3.32}\]
	Now, combining (3.31) and (3.32), we conclude that
	\[\begin{array}{*{20}{l}}
	{{\omega _1}\varphi \left( {{{\bf{m}}^{k + 1}}} \right) + \frac{\rho }{2}\left\| {{{\bf{m}}^{k + 1}} - {\nabla ^2}{i^{k + 1}} + \frac{{{\bf{y}}_2^k}}{\rho }} \right\|_2^2}\\
	{\;\;\; \le {\omega _1}\varphi \left( {{{\bf{m}}^k}} \right) + \frac{\rho }{2}\left\| {{{\bf{m}}^k} - {\nabla ^2}{i^{k + 1}} + \frac{{{\bf{y}}_2^k}}{\rho }} \right\|_2^2 - \frac{{{\eta _2}{\omega _1}}}{2}\left\| {{{\bf{m}}^{k + 1}} - {{\bf{m}}^k}} \right\|_2^2,}
	\end{array}\]
	which implies that
	\[\begin{array}{*{20}{l}}
	{L\left( {{i^{k + 1}},{r^{k + 1}},{v^{k{\rm{ + 1}}}},{u^{k + 1}},{{\bf{m}}^{k + 1}},{{\bf{n}}^k};y_1^k,{\bf{y}}_2^k,{\bf{y}}_3^k} \right)}\\
	{\;\; - L\left( {{i^{k + 1}},{r^{k + 1}},{v^{k + 1}},{u^{k + 1}},{{\bf{m}}^k},{{\bf{n}}^k};y_1^k,{\bf{y}}_2^k,{\bf{y}}_3^k} \right) \le  - \frac{{{\eta _2}{\omega _1}}}{2}\left\| {{{\bf{m}}^{k + 1}} - {{\bf{m}}^k}} \right\|_2^2.}
	\end{array}\]
	
	 Similarly, we can infer that there must be a positive constant $\eta_3$ such that
	 \[\begin{array}{*{20}{l}}
	 {L\left( {{i^{k + 1}},{r^{k + 1}},{v^{k + 1}},{u^{k + 1}},{{\bf{m}}^{k + 1}},{{\bf{n}}^{k + 1}};y_1^k,{\bf{y}}_2^k,{\bf{y}}_3^k} \right)}\\
	 {\;\; - L\left( {{i^{k + 1}},{r^{k + 1}},{v^{k + 1}},{u^{k + 1}},{{\bf{m}}^{k + 1}},{{\bf{n}}^k};y_1^k,{\bf{y}}_2^k,{\bf{y}}_3^k} \right) \le  - \frac{{{\eta _3}{\omega _2}}}{2}\left\| {{{\bf{n}}^{k + 1}} - {{\bf{n}}^k}} \right\|_2^2.}
	 \end{array}\]
	
	 The desired result is obtained.\qed
\end{proof}

\begin{theorem}
	Let the sequence $ \left( {{i^k},{r^k},{v^k},{u^k},{{\bf{m}}^k},{{\bf{n}}^k};y_1^k,{\bf{y}}_2^k,{\bf{y}}_3^k} \right)$ generated by Algorithm 1, and we can infer that
	\[\begin{array}{*{20}{l}}
	{L\left( {{i^{k + 1}},{r^{k + 1}},{v^{k + 1}},{u^{k + 1}},{{\bf{m}}^{k + 1}},{{\bf{n}}^{k + 1}};y_1^{k + 1},{\bf{y}}_2^{k + 1},{\bf{y}}_3^{k + 1}} \right)}\\
	{\;\;\; - L\left( {{i^{k + 1}},{r^{k + 1}},{v^{k + 1}},{u^{k + 1}},{{\bf{m}}^{k + 1}},{{\bf{n}}^{k + 1}};y_1^k,{\bf{y}}_2^k,{\bf{y}}_3^k} \right)}\\
	{\;\;\;\;\;\; \le \frac{{{\lambda ^2}}}{\rho }\left\| {{\Delta ^{ - 1}}} \right\|_2^2\left\| {{u^{k + 1}} - {u^k}} \right\|_2^2 + \frac{{{K^2}}}{\rho }\left( {\omega _1^2\left\| {{{\bf{m}}^{k + 1}} - {{\bf{m}}^k}} \right\|_2^2 + \omega _2^2\left\| {{{\bf{n}}^{k + 1}} - {{\bf{n}}^k}} \right\|_2^2} \right)}
	\end{array} \tag{3.33}\]
\end{theorem}

\begin{proof}
	In the $u$ subproblem, since $u^{k+1}$ is a minimum point of the augmented Lagrangian function $L\left( {{i^{k+1}},{r^{k+1}},{v^{k+1}},u,{{\bf{m}}^k},{{\bf{n}}^k};y_1^k,{\bf{y}}_2^k,{\bf{y}}_3^k} \right)$, according to the necessary condition of the extreme point, we have
	\[- \lambda {\Delta ^{ - 1}}\left( {{u^{k + 1}} - f} \right) + \rho \left( {{u^{k + 1}} - {e^{{v^{k + 1}}}} + \frac{{y_1^k}}{\rho }} \right) = 0.\]
	Furthermore, according to the updated form of Lagrange multiplier Eqs(3.25), we obtain
	\[y_1^{k + 1} = y_1^k + \rho \left( {{u^{k + 1}} - {e^{{v^{k + 1}}}}} \right) = \lambda {\Delta ^{ - 1}}\left( {{u^{k + 1}} - f} \right).\tag{3.34}\]
	Therefore, we can easily derive the inequality
	\[\left\| {y_1^{k + 1} - y_1^k} \right\|_2^2 = \left\| {\lambda {\Delta ^{ - 1}}\left( {{u^{k + 1}} - {u^k}} \right)} \right\|_2^2 \le {\lambda ^2}\left\| {{\Delta ^{ - 1}}} \right\|_2^2\left\| {{u^{k + 1}} - {u^k}} \right\|_2^2,\tag{3.35}\]
	where $\left\| {{\Delta ^{ - 1}}} \right\|_2^2$ is an operator norm. Similarly, since $i^{k+1}$ and $r^{k+1}$ are solutions of the minimization problems (3.21) and (3.23), respectively, we have
	\[{{\omega _1}\nabla \varphi \left( {{{\bf{m}}^{k + 1}}} \right) + \rho \left( {{{\bf{m}}^{k + 1}} - {\nabla ^2}{i^{k + 1}} + \frac{{{\bf{y}}_2^k}}{\rho }} \right) = 0,}\]
	and
	\[{{\omega _2}\nabla \varphi \left( {{{\bf{n}}^{k + 1}}} \right) + \rho \left( {{{\bf{n}}^{k + 1}} - \nabla {r^{k + 1}} + \frac{{{\bf{y}}_3^k}}{\rho }} \right) = 0.}\]
	Combining the updated form of Lagrange multiplier Eqs(3.25), we obtain
	\[\left\{ \begin{array}{l}
	{\bf{y}}_2^{k + 1} = {\bf{y}}_2^k + \rho \left( {{{\bf{m}}^{k + 1}} - {\nabla ^2}{i^{k + 1}}} \right) =  - {\omega _1}\nabla \varphi \left( {{{\bf{m}}^{k + 1}}} \right)\\
	{\bf{y}}_3^{k + 1} = {\bf{y}}_3^k + \rho \left( {{{\bf{n}}^{k + 1}} - \nabla {r^{k + 1}}} \right) =  - {\omega _2}\nabla \varphi \left( {{{\bf{n}}^{k + 1}}} \right)
	\end{array} \right..\tag{3.36}\]
	And combining the Assumption 3.2, we easily derive the inequality system
	\[\left\{ \begin{array}{l}
	\left\| {{\bf{y}}_2^{k + 1} - {\bf{y}}_2^k} \right\|_2^2 = \left\| {{\omega _1}\nabla \varphi \left( {{{\bf{m}}^{k + 1}}} \right) - {\omega _1}\nabla \varphi \left( {{{\bf{m}}^k}} \right)} \right\|_2^2\; \le \omega _1^2{K^2}\left\| {{{\bf{m}}^{k + 1}} - {{\bf{m}}^k}} \right\|_2^2\\
	\left\| {{\bf{y}}_3^{k + 1} - {\bf{y}}_3^k} \right\|_2^2 = \left\| {{\omega _2}\nabla \varphi \left( {{{\bf{n}}^{k + 1}}} \right) - {\omega _2}\nabla \varphi \left( {{{\bf{n}}^k}} \right)} \right\|_2^2 \le \omega _2^2{K^2}\left\| {{{\bf{n}}^{k + 1}} - {{\bf{n}}^k}} \right\|_2^2
	\end{array} \right..\tag{3.37}\]
	
	Finally, combining (3.25), (3.35) and (3.37), we obtain
	\[\begin{array}{*{20}{l}}
	{L\left( {{i^{k + 1}},{r^{k + 1}},{v^{k + 1}},{u^{k + 1}},{{\bf{m}}^{k + 1}},{{\bf{n}}^{k + 1}};y_1^{k + 1},{\bf{y}}_2^{k + 1},{\bf{y}}_3^{k + 1}} \right)}\\
	{\;\;\; - L\left( {{i^{k + 1}},{r^{k + 1}},{v^{k + 1}},{u^{k + 1}},{{\bf{m}}^{k + 1}},{{\bf{n}}^{k + 1}};y_1^k,{\bf{y}}_2^k,{\bf{y}}_3^k} \right)}\\
	{\;\;\; = \left\langle {y_1^{k + 1} - y_1^k,{u^{k + 1}} - {e^{{v^{k + 1}}}}} \right\rangle  + \left\langle {{\bf{y}}_2^{k + 1} - {\bf{y}}_2^k,{{\bf{m}}^{k + 1}} - {\nabla ^2}{i^{k + 1}}} \right\rangle  + \left\langle {{\bf{y}}_3^{k + 1} - {\bf{y}}_3^k,{{\bf{n}}^{k + 1}} - \nabla {r^{k + 1}}} \right\rangle }\\
	{\;\;\; = \frac{1}{\rho }\left( {\left\| {y_1^{k + 1} - y_1^k} \right\|_2^2 + \left\| {{\bf{y}}_2^{k + 1} - {\bf{y}}_2^k} \right\|_2^2 + \left\| {{\bf{y}}_3^{k + 1} - {\bf{y}}_3^k} \right\|_2^2} \right)}\\
	{\;\;\; \le \frac{{{\lambda ^2}}}{\rho }\left\| {{\Delta ^{ - 1}}} \right\|_2^2\left\| {{u^{k + 1}} - {u^k}} \right\|_2^2 + \frac{{{K^2}}}{\rho }\left( {\omega _1^2\left\| {{{\bf{m}}^{k + 1}} - {{\bf{m}}^k}} \right\|_2^2 + \omega _2^2\left\| {{{\bf{n}}^{k + 1}} - {{\bf{n}}^k}} \right\|_2^2} \right).}
	\end{array}\]
	
	The desired result is obtained.\qed
\end{proof}

\begin{theorem}
	Let the sequence $ \left( {{i^k},{r^k},{v^k},{u^k},{{\bf{m}}^k},{{\bf{n}}^k};y_1^k,{\bf{y}}_2^k,{\bf{y}}_3^k} \right)$ generated by Algorithm 1. If the parameter $\rho$ is large enough such that $\rho  > \max \left\{ {{{2{\lambda ^2}\left\| {{\Delta ^{ - 1}}} \right\|_2^2} \mathord{\left/
				{\vphantom {{2{\lambda ^2}\left\| {{\Delta ^{ - 1}}} \right\|_2^2} {{\eta _1}}}} \right.
				\kern-\nulldelimiterspace} {{\eta _1}}},\;{{2{K^2}{\omega _1}} \mathord{\left/
				{\vphantom {{2{K^2}{\omega _1}} {{\eta _2}}}} \right.
				\kern-\nulldelimiterspace} {{\eta _2}}},\;{{2{K^2}{\omega _2}} \mathord{\left/
				{\vphantom {{2{K^2}{\omega _2}} {{\eta _3}}}} \right.
				\kern-\nulldelimiterspace} {{\eta _3}}}} \right\}$, and we can infer that
	\[\begin{array}{l}
	L\left( {{i^{k + 1}},{r^{k + 1}},{v^{k + 1}},{u^{k + 1}},{{\bf{m}}^{k + 1}},{{\bf{n}}^{k + 1}};y_1^{k + 1},{\bf{y}}_2^{k + 1},{\bf{y}}_3^{k + 1}} \right)\\
	\;\;\; \le L\left( {{i^k},{r^k},{v^k},{u^k},{{\bf{m}}^k},{{\bf{n}}^k};y_1^k,{\bf{y}}_2^k,{\bf{y}}_3^k} \right).
	\end{array}
	\]
\end{theorem}

\begin{proof}
	During the iteration of the proposed algorithm, the difference of the augmented Lagrangian function $L$ can be rewritten as
	\[\begin{array}{l}
	L\left( {{i^{k + 1}},{r^{k + 1}},{v^{k + 1}},{u^{k + 1}},{{\bf{m}}^{k + 1}},{{\bf{n}}^{k + 1}};y_1^{k + 1},{\bf{y}}_2^{k + 1},{\bf{y}}_3^{k + 1}} \right)\;\\
	\;\;\; - L\left( {{i^k},{r^k},{v^k},{u^k},{{\bf{m}}^k},{{\bf{n}}^k};y_1^k,{\bf{y}}_2^k,{\bf{y}}_3^k} \right)\\
	\;\;\; = L\left( {{i^{k + 1}},{r^{k + 1}},{v^{k + 1}},{u^{k + 1}},{{\bf{m}}^{k + 1}},{{\bf{n}}^{k + 1}};y_1^{k + 1},{\bf{y}}_2^{k + 1},{\bf{y}}_3^{k + 1}} \right)\\
	\;\;\;\;\;\; - L\left( {{i^{k + 1}},{r^{k + 1}},{v^{k + 1}},{u^{k + 1}},{{\bf{m}}^{k + 1}},{{\bf{n}}^{k + 1}};y_1^k,{\bf{y}}_2^k,{\bf{y}}_3^k} \right) +  \cdots \; \cdots \\
	\;\;\; + L\left( {{i^{k + 1}},{r^k},{v^k},{u^k},{{\bf{m}}^k},{{\bf{n}}^k};y_1^k,{\bf{y}}_2^k,{\bf{y}}_3^k} \right) - L\left( {{i^k},{r^k},{v^k},{u^k},{{\bf{m}}^k},{{\bf{n}}^k};y_1^k,{\bf{y}}_2^k,{\bf{y}}_3^k} \right).
	\end{array}\]
	According to Theorem 3.5-3.7, we derive
	\[  \begin{array}{l}
	L\left( {{i^{k + 1}},{r^{k + 1}},{v^{k + 1}},{u^{k + 1}},{{\bf{m}}^{k + 1}},{{\bf{n}}^{k + 1}};y_1^{k + 1},{\bf{y}}_2^{k + 1},{\bf{y}}_3^{k + 1}} \right)\\
	\;\;\; - L\left( {{i^k},{r^k},{v^k},{u^k},{{\bf{m}}^k},{{\bf{n}}^k};y_1^k,{\bf{y}}_2^k,{\bf{y}}_3^k} \right)\\
	\;\;\; \le  - \frac{{{\gamma _1}}}{2}\left\| {{i^{k + 1}} - {i^k}} \right\|_2^2 - \frac{{{\gamma _2}}}{2}\left\| {{r^{k + 1}} - {r^k}} \right\|_2^2 - \frac{{{\gamma _3}}}{2}\left\| {{v^{k + 1}} - {v^k}} \right\|_2^2\\
	\;\;\;\;\;\; - \left( {\frac{{{\eta _1}}}{2} - \frac{{{\lambda ^2}}}{\rho }\left\| {{\Delta ^{ - 1}}} \right\|_2^2} \right)\left\| {{u^{k + 1}} - {u^k}} \right\|_2^2 - \left( {\frac{{{\eta _2}{\omega _1}}}{2} - \frac{{{K^2}\omega _1^2}}{\rho }} \right)\left\| {{{\bf{m}}^{k + 1}} - {{\bf{m}}^k}} \right\|_2^2\\
	\;\;\;\;\;\; - \left( {\frac{{{\eta _3}{\omega _2}}}{2} - \frac{{{K^2}\omega _2^2}}{\rho }} \right)\left\| {{{\bf{n}}^{k + 1}} - {{\bf{n}}^k}} \right\|_2^2.
	\end{array}\tag{3.38}\]
	Obviously, if $\rho  > \max \left\{ {{{2{\lambda ^2}\left\| {{\Delta ^{ - 1}}} \right\|_2^2} \mathord{\left/
				{\vphantom {{2{\lambda ^2}\left\| {{\Delta ^{ - 1}}} \right\|_2^2} {{\eta _1}}}} \right.
				\kern-\nulldelimiterspace} {{\eta _1}}},\;{{2{K^2}{\omega _1}} \mathord{\left/
				{\vphantom {{2{K^2}{\omega _1}} {{\eta _2}}}} \right.
				\kern-\nulldelimiterspace} {{\eta _2}}},\;{{2{K^2}{\omega _2}} \mathord{\left/
				{\vphantom {{2{K^2}{\omega _2}} {{\eta _3}}}} \right.
				\kern-\nulldelimiterspace} {{\eta _3}}}} \right\}$, we obtain
	\[\begin{array}{l}
	L\left( {{i^{k + 1}},{r^{k + 1}},{v^{k + 1}},{u^{k + 1}},{{\bf{m}}^{k + 1}},{{\bf{n}}^{k + 1}};y_1^{k + 1},{\bf{y}}_2^{k + 1},{\bf{y}}_3^{k + 1}} \right)\\
	\;\;\; \le L\left( {{i^k},{r^k},{v^k},{u^k},{{\bf{m}}^k},{{\bf{n}}^k};y_1^k,{\bf{y}}_2^k,{\bf{y}}_3^k} \right).
	\end{array}\]
	
	The desired result is obtained.\qed
\end{proof}

\begin{theorem}
	The augmented Lagrangian function $L\left( {{i^k},{r^k},{v^k},{u^k},{{\bf{m}}^k},{{\bf{n}}^k};y_1^k,{\bf{y}}_2^k,{\bf{y}}_3^k} \right)$ is bounded below, and it is convergent as $k \to  + \infty $.
\end{theorem}

\begin{proof}
	The augmented Lagrangian function $L$ can be rewritten as
	\[\begin{array}{*{20}{l}}
	{L\left( {{i^{k + 1}},{r^{k + 1}},{v^{k + 1}},{u^{k + 1}},{{\bf{m}}^{k + 1}},{{\bf{n}}^{k + 1}};y_1^{k + 1},{\bf{y}}_2^{k + 1},{\bf{y}}_3^{k + 1}} \right)}\\
	{\;\;\; = \frac{\lambda }{2}\left\| {\nabla \left( {{\Delta ^{ - 1}}\left( {f - {u^{k + 1}}} \right)} \right)} \right\|_2^2 + \frac{\theta }{2}\left\| {{i^{k + 1}}} \right\|_2^2 + \frac{\beta }{2}\left\| {{v^{k + 1}} - {i^{k + 1}} - {r^{k + 1}}} \right\|_2^2}\\
	{\begin{array}{*{20}{l}}
		{\;\;\; + {\omega _1}\int_\Omega  {\varphi \left( {{{\bf{m}}^{k + 1}}} \right)} dx + {\omega _2}\int_\Omega  {\varphi \left( {{{\bf{n}}^{k + 1}}} \right)} dx}\\
		{\;\;\; + \left\langle {y_1^{k + 1},{u^{k + 1}} - {e^{{v^{k + 1}}}}} \right\rangle  + \frac{\rho }{2}\left\| {{u^{k + 1}} - {e^{{v^{k + 1}}}}} \right\|_2^2}
		\end{array}}\\
	{\;\;\; + \left\langle {{\bf{y}}_2^{k + 1},{{\bf{m}}^{k + 1}} - {\nabla ^2}{i^{k + 1}}} \right\rangle  + \frac{\rho }{2}\left\| {{{\bf{m}}^{k + 1}} - {\nabla ^2}{i^{k + 1}}} \right\|_2^2}\\
	{\;\;\; + \left\langle {{\bf{y}}_3^{k + 1},{{\bf{n}}^{k + 1}} - \nabla {r^{k + 1}}} \right\rangle  + \frac{\rho }{2}\left\| {{{\bf{n}}^{k + 1}} - \nabla {r^{k + 1}}} \right\|_2^2.}
	\end{array}\]
	Furthermore, combining (3.34) and (3.36), we derive the inequality system as follows
	\[\begin{array}{l}
	L\left( {{i^{k + 1}},{r^{k + 1}},{v^{k + 1}},{u^{k + 1}},{{\bf{m}}^{k + 1}},{{\bf{n}}^{k + 1}};y_1^{k + 1},{\bf{y}}_2^{k + 1},{\bf{y}}_3^{k + 1}} \right)\\
	\;\;\; = \frac{\lambda }{2}\left\| {\nabla \left( {{\Delta ^{ - 1}}\left( {f - {u^{k + 1}}} \right)} \right)} \right\|_2^2 + \frac{\theta }{2}\left\| {{i^{k + 1}}} \right\|_2^2 + \frac{\beta }{2}\left\| {{v^{k + 1}} - {i^{k + 1}} - {r^{k + 1}}} \right\|_2^2\\
	\;\;\; + {\omega _1}\int_\Omega  {\varphi \left( {{{\bf{m}}^{k + 1}}} \right)} dx + {\omega _2}\int_\Omega  {\varphi \left( {{{\bf{n}}^{k + 1}}} \right)} dx\\
	\;\;\; + \left\langle {\lambda {\Delta ^{ - 1}}\left( {{u^{k + 1}} - f} \right),{u^{k + 1}} - {e^{{v^{k + 1}}}}} \right\rangle  + \frac{\rho }{2}\left\| {{u^{k + 1}} - {e^{{v^{k + 1}}}}} \right\|_2^2\\
	\;\;\; + \left\langle { - {\omega _1}\nabla \varphi \left( {{{\bf{m}}^{k + 1}}} \right),{{\bf{m}}^{k + 1}} - {\nabla ^2}{i^{k + 1}}} \right\rangle  + \frac{\rho }{2}\left\| {{{\bf{m}}^{k + 1}} - {\nabla ^2}{i^{k + 1}}} \right\|_2^2\\
	\;\;\; + \left\langle { - {\omega _2}\nabla \varphi \left( {{{\bf{n}}^{k + 1}}} \right),{{\bf{n}}^{k + 1}} - \nabla {r^{k + 1}}} \right\rangle  + \frac{\rho }{2}\left\| {{{\bf{n}}^{k + 1}} - \nabla {r^{k + 1}}} \right\|_2^2.
	\end{array}\tag{3.39}\]
	Since $\varphi \left(  \cdot  \right)$ is a nonconvex function, we have
	\[ \left\{ \begin{array}{l}
	\varphi \left( {{{\bf{m}}^{k + 1}}} \right) + \left\langle {\nabla \varphi \left( {{{\bf{m}}^{k + 1}}} \right),{\nabla ^2}{i^{k + 1}} - {{\bf{m}}^{k + 1}}} \right\rangle  \ge \varphi \left( {{\nabla ^2}{i^{k + 1}}} \right)\\
	\varphi \left( {{{\bf{n}}^{k + 1}}} \right) + \left\langle {\nabla \varphi \left( {{{\bf{n}}^{k + 1}}} \right),\nabla {r^{k + 1}} - {{\bf{n}}^{k + 1}}} \right\rangle  \ge \varphi \left( {\nabla {r^{k + 1}}} \right)
	\end{array} \right..\]
	By above discussion, we obtain
	\[\begin{array}{*{20}{l}}
	{L\left( {{i^{k + 1}},{r^{k + 1}},{v^{k + 1}},{u^{k + 1}},{{\bf{m}}^{k + 1}},{{\bf{n}}^{k + 1}};y_1^{k + 1},{\bf{y}}_2^{k + 1},{\bf{y}}_3^{k + 1}} \right)}\\
	{\;\;\; \ge \frac{\lambda }{2}\left\| {f - {e^{{i^{k + 1}} + {r^{k + 1}}}}} \right\|_{{H^{ - 1}}}^2 + {\omega _1}\int_\Omega  {\varphi \left( {{\nabla ^{\rm{2}}}{i^{k + 1}}} \right)} dx + {\omega _2}\int_\Omega  {\varphi \left( {\nabla {r^{k + 1}}} \right)} dx}\\
	{\;\;\; + \frac{\theta }{2}\left\| {{i^{k + 1}}} \right\|_2^2 + \frac{\rho }{2}\left\| {{u^{k + 1}} - {e^{{v^{k + 1}}}}} \right\|_2^2 + \frac{\rho }{2}\left\| {{{\bf{m}}^{k + 1}} - {\nabla ^2}{i^{k + 1}}} \right\|_2^2}\\
	{\;\;\; + \frac{\rho }{2}\left\| {{{\bf{n}}^{k + 1}} - \nabla {r^{k + 1}}} \right\|_2^2 \ge E\left( {{i^{k + 1}},{r^{k + 1}}} \right) \ge \underline E .}
	\end{array}\tag{3.40}\]
	Finally, the augmented Lagrangian function $L\left( {{i^k},{r^k},{v^k},{u^k},{{\bf{m}}^k},{{\bf{n}}^k};y_1^k,{\bf{y}}_2^k,{\bf{y}}_3^k} \right)$ is monotone decreasing according to Theorem 3.8, so the function $L$ is convergent as $k \to  + \infty $.
	
	The desired result is obtained.\qed
\end{proof}

\begin{theorem}
	Let the sequence $ \left( {{i^k},{r^k},{v^k},{u^k},{{\bf{m}}^k},{{\bf{n}}^k};y_1^k,{\bf{y}}_2^k,{\bf{y}}_3^k} \right)$ generated by Algorithm 1. If the parameter $\rho$ is large enough such that $\rho  > \max \left\{ {{{2{\lambda ^2}\left\| {{\Delta ^{ - 1}}} \right\|_2^2} \mathord{\left/
			{\vphantom {{2{\lambda ^2}\left\| {{\Delta ^{ - 1}}} \right\|_2^2} {{\eta _1}}}} \right.
			\kern-\nulldelimiterspace} {{\eta _1}}},\;{{2{K^2}{\omega _1}} \mathord{\left/
			{\vphantom {{2{K^2}{\omega _1}} {{\eta _2}}}} \right.
			\kern-\nulldelimiterspace} {{\eta _2}}},\;{{2{K^2}{\omega _2}} \mathord{\left/
			{\vphantom {{2{K^2}{\omega _2}} {{\eta _3}}}} \right.
			\kern-\nulldelimiterspace} {{\eta _3}}}} \right\}$, and we can infer that
	$$\mathop {\lim }\limits_{k \to \infty } \left\| {{i^{k + 1}} - {i^k}} \right\|_2^2 = 0,\;\;\mathop {\lim }\limits_{k \to \infty } \left\| {{r^{k + 1}} - {r^k}} \right\|_2^2 = 0,\;\;\mathop {\lim }\limits_{k \to \infty } \left\| {{v^{k + 1}} - {v^k}} \right\|_2^2 = 0,$$
	$$\mathop {\lim }\limits_{k \to \infty } \left\| {{u^{k + 1}} - {u^k}} \right\|_2^2 = 0,\;\;\mathop {\lim }\limits_{k \to \infty } \left\| {{{\bf{m}}^{k + 1}} - {{\bf{m}}^k}} \right\|_2^2 = 0,\;\;\mathop {\lim }\limits_{k \to \infty } \left\| {{{\bf{n}}^{k + 1}} - {{\bf{n}}^k}} \right\|_2^2 = 0,$$
	$$\mathop {\lim }\limits_{k \to \infty } \left\| {{u^{k + 1}} - {e^{{v^{k + 1}}}}} \right\|_2^2 = 0,\;\;\mathop {\lim }\limits_{k \to \infty } \left\| {{{\bf{m}}^{k + 1}} - {\nabla ^2}{i^{k + 1}}} \right\|_2^2 = 0,\;\;\mathop {\lim }\limits_{k \to \infty } \left\| {{{\bf{n}}^{k + 1}} - \nabla {r^{k + 1}}} \right\|_2^2 = 0.$$
\end{theorem}

\begin{proof}
	Firstly, combining Theorem 3.9 and taking the limit of Eq(3.38), we obtain
	$$\mathop {\lim }\limits_{k \to \infty } \left\| {{i^{k + 1}} - {i^k}} \right\|_2^2 = 0,\;\;\mathop {\lim }\limits_{k \to \infty } \left\| {{r^{k + 1}} - {r^k}} \right\|_2^2 = 0,\;\;\mathop {\lim }\limits_{k \to \infty } \left\| {{v^{k + 1}} - {v^k}} \right\|_2^2 = 0,$$
	$$\mathop {\lim }\limits_{k \to \infty } \left\| {{u^{k + 1}} - {u^k}} \right\|_2^2 = 0,\;\;\mathop {\lim }\limits_{k \to \infty } \left\| {{{\bf{m}}^{k + 1}} - {{\bf{m}}^k}} \right\|_2^2 = 0,\;\;\mathop {\lim }\limits_{k \to \infty } \left\| {{{\bf{n}}^{k + 1}} - {{\bf{n}}^k}} \right\|_2^2 = 0.$$
	Next, combining (3.35) and (3.37), we have
	\[ \left\{ \begin{array}{l}
	\mathop {\lim }\limits_{k \to \infty } \left\| {y_1^{k + 1} - y_1^k} \right\|_2^2 \le {\lambda ^2}\left\| {{\Delta ^{ - 1}}} \right\|_2^2\mathop {\lim }\limits_{k \to \infty } \left\| {{u^{k + 1}} - {u^k}} \right\|_2^2 = 0\\
	\mathop {\lim }\limits_{k \to \infty } \left\| {{\bf{y}}_2^{k + 1} - {\bf{y}}_2^k} \right\|_2^2 \le \omega _1^2{K^2}\mathop {\lim }\limits_{k \to \infty } \left\| {{{\bf{m}}^{k + 1}} - {{\bf{m}}^k}} \right\|_2^2 = 0\\
	\mathop {\lim }\limits_{k \to \infty } \left\| {{\bf{y}}_3^{k + 1} - {\bf{y}}_3^k} \right\|_2^2 \le \omega _2^2{K^2}\mathop {\lim }\limits_{k \to \infty } \left\| {{{\bf{n}}^{k + 1}} - {{\bf{n}}^k}} \right\|_2^2 = 0
	\end{array} \right..\tag{3.41}\]
	Finally, combining the updated format of Lagrange multiplier ${y_1},\;{{\bf{y}}_2}$ and ${{\bf{y}}_{\bf{3}}}$ in Eqs(3.25), we can easily deduce
	\[\left\{ \begin{array}{l}
	\mathop {\lim }\limits_{k \to \infty } \left\| {{u^{k + 1}} - {e^{{v^{k + 1}}}}} \right\|_2^2 = \frac{1}{{{\rho ^2}}}\mathop {\lim }\limits_{k \to \infty } \left\| {y_1^{k + 1} - y_1^k} \right\|_2^2 = 0\\
	\mathop {\lim }\limits_{k \to \infty } \left\| {{{\bf{m}}^{k + 1}} - {\nabla ^2}{i^{k + 1}}} \right\|_2^2 = \frac{1}{{{\rho ^2}}}\mathop {\lim }\limits_{k \to \infty } \left\| {{\bf{y}}_2^{k + 1} - {\bf{y}}_2^k} \right\|_2^2 = 0\\
	\mathop {\lim }\limits_{k \to \infty } \left\| {{{\bf{n}}^{k + 1}} - \nabla {r^{k + 1}}} \right\|_2^2 = \frac{1}{{{\rho ^2}}}\mathop {\lim }\limits_{k \to \infty } \left\| {{\bf{y}}_3^{k + 1} - {\bf{y}}_3^k} \right\|_2^2 = 0
	\end{array} \right..\]
	
	The desired result is obtained.\qed
\end{proof}

\begin{theorem}
	Let $i \in {X_1},\;r \in {X_2},\;v \in {X_3},\;u = {e^v} \in {G_1},\;{\bf{m}} = {\nabla ^2}i \in {G_2},\;{\bf{n}} = \nabla r \in {G_3}$. If ${X_1},\;{X_2},\;{X_3},\;{G_1},\;{G_2},\;{G_3}$ are the compact sets, the sequence ${z^k} = \left( {{i^k},{r^k},{v^k},{u^k},{{\bf{m}}^k},{{\bf{n}}^k};y_1^k,{\bf{y}}_2^k,{\bf{y}}_3^k} \right)$ generated by Algorithm 1 converges to a limit point ${z^*} = \left( {{i^*},{r^*},{v^*},{u^*},{{\bf{m}}^*},{{\bf{n}}^*};y_1^*,{\bf{y}}_2^*,{\bf{y}}_3^*} \right)$, which is a stationary point of the augmented Lagrangian function $ L\left( {{i},{r},{v},{u},{{\bf{m}}},{{\bf{n}}};y_1,{\bf{y}}_2,{\bf{y}}_3} \right)$.
\end{theorem}

\begin{proof}
	Since ${X_1},\;{X_2},\;{X_3},\;{G_1},\;{G_2},\;{G_3}$ are compact sets, according to Theorem 3.10, then we can infer that there must be a subsequence $\left\{ {{z^{{k_i}}}} \right\}$ of $\left\{ {{z^k}} \right\}$ such that  ${z^{{k_i}}} \to {z^*}$ as ${k_i} \to  + \infty $.
	
	Next, we prove that the point ${z^*}$ is a stationary point of the augmented Lagrangian function $L\left( {{i^k},{r^k},{v^k},{u^k},{{\bf{m}}^k},{{\bf{n}}^k};y_1^k,{\bf{y}}_2^k,{\bf{y}}_3^k} \right)$. The sequence ${z^*}$ must satisfy the sufficient condition of the minimization problem, so we obtain
	$$\beta \left( {{v^k} - {i^{k + 1}} - {r^k}} \right) + \rho {\left( {{\nabla ^2}} \right)^*}\left( {{{\bf{m}}^k} - {\nabla ^2}{i^{k + 1}} + \frac{{{\bf{y}}_2^k}}{\rho }} \right) + \theta {i^{k + 1}} = 0,$$
	$$\beta \left( {{v^k} - {i^{k + 1}} - {r^{k + 1}}} \right) + \rho {\nabla ^*}\left( {{{\bf{n}}^k} - \nabla {r^{k + 1}} + \frac{{{\bf{y}}_3^k}}{\rho }} \right) = 0,$$
	$$\beta \left( {{v^{k + 1}} - {i^{k + 1}} - {r^{k + 1}}} \right) - \rho {e^{{v^{k + 1}}}}\left( {{u^k} - {e^{{v^{k + 1}}}} + \frac{{y_1^k}}{\rho }} \right) = 0,$$
	$$ - \lambda {\Delta ^{ - 1}}\left( {{u^{k + 1}} - f} \right) + \rho \left( {{u^{k + 1}} - {e^{{v^{k + 1}}}} + \frac{{y_1^k}}{\rho }} \right) = 0,$$
	$${\omega _1}\nabla \varphi \left( {{{\bf{m}}^{k + 1}}} \right) + \rho \left( {{{\bf{m}}^{k + 1}} - {\nabla ^2}{i^{k + 1}} + \frac{{{\bf{y}}_2^k}}{\rho }} \right) = 0,$$
	$${\omega _2}\nabla \varphi \left( {{{\bf{n}}^{k + 1}}} \right) + \rho \left( {{{\bf{n}}^{k + 1}} - \nabla {r^{k + 1}} + \frac{{{\bf{y}}_3^k}}{\rho }} \right) = 0.$$
	
	Finally, combining Theorem 3.10, passing the limit in the above six equations along the subsequence $\left\{ {{z^{{k_i}}}} \right\}$, we have
	$${\left( {{\nabla ^2}} \right)^*}{\bf{y}}_2^* =  - \beta \left( {{v^*} - {i^*} - {r^*}} \right) - \theta {i^*},\;{\nabla ^*}{\bf{y}}_3^* =  - \beta \left( {{v^*} - {i^*} - {r^*}} \right),\;{e^{{v^*}}}y_1^* = \beta \left( {{v^*} - {i^*} - {r^*}} \right),$$
	$$y_1^* = \lambda {\Delta ^{ - 1}}\left( {{u^*} - f} \right),\;{\bf{y}}_2^* =  - {\omega _1}\nabla \varphi \left( {{{\bf{m}}^{*}}} \right),\;{\bf{y}}_3^* =  - {\omega _2}\nabla \varphi \left( {{{\bf{n}}^{*}}} \right),$$
	$${u^*} = {e^{{v^*}}},\;\;{{\bf{m}}^*} = {\nabla ^2}{i^*},\;{{\bf{n}}^*} = \nabla {r^*}.$$
	which implies that the point ${z^*}$ is a stationary point of the function $L\left( {i,r,v,u,{\bf{m}},{\bf{n}};{y_1},{{\bf{y}}_2},{{\bf{y}}_3}} \right)$.
	
		The desired result is obtained.\qed
\end{proof}

\section{Numerical experiment}

In this section, we conduct a series of numerical experiments to further illustrate the superiority and effectiveness of the proposed model from four different perspectives. Firstly, we verify the robustness of the selection of the initial value ${u^0}$ in Algorithm 1 through experimental results. Secondly, we perform numerical experiments to test two commonly used generalized nonconvex functions mentioned in Subsection 2.2. Throughout these experiments, unless stated otherwise, the generalized nonconvex function in Algorithm 1 is set as $\varphi \left( { t } \right) = {\left| t \right|^p}\left( {0< p < 1} \right)$. We then proceed to experiment and demonstrate the ternary decomposition results of two test images under both noise-free and noisy experimental environments. Finally, we compare the proposed model with several advanced denoising models on specific test image sets to clearly illustrate its effectiveness in denoising applications

In the numerical experiments, the iteration stop condition of the proposed algorithm is that the maximum number of iterations ${N_{\max }} = 1000$ or the relative error of $u$ is less than $tol = {10^{ - 5}}$. Here the relative error is defined as
\[R\left( {{u^k}} \right) = \frac{{{{\left\| {{u^k} - {u^{k - 1}}} \right\|}_2}}}{{{{\left\| {{u^k}} \right\|}_2}}}.\]

In addition, the PSNR \cite{b36} and MSSIM \cite{b37} are selected as the quantitative evaluation indices for assessing the quality of the restored images. It is well known that higher values of PSNR and MSSIM indicate better image quality. It should be noted that Fig.2 displays some of the test images used in the numerical experiments, with a size of $256 \times 256$. All numerical experiments in this paper were implemented in MATLAB and executed on a PC equipped with Windows 10, an Intel Core i5-8300H CPU running at 2.3GHz, and 16GB of RAM.

\begin{center}
	\begin{minipage}
		{1\textwidth}\centering {\includegraphics[width=13cm,height=5.5cm]{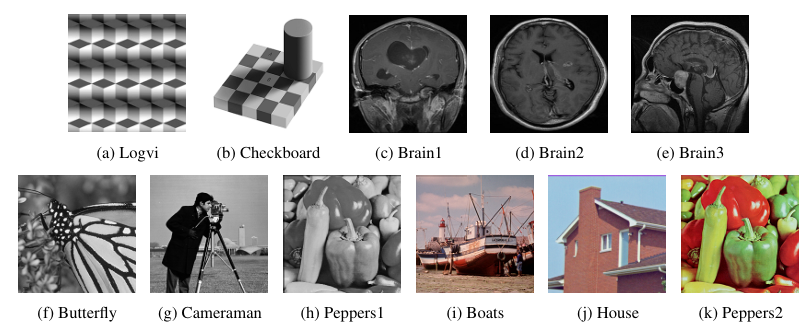}}
	\end{minipage}
	\\[1mm]
	\vskip2mm
	{\small \textbf{Fig.2} Some test images for numerical experiments.}
\end{center}

\subsection{Sensitivity analysis of initial value}

This example aims to demonstrate the robustness of the selection of  initial values of ${u^0}$ in the algorithm based on the experimental results. The proposed algorithm in this paper requires initializing certain parameters and variables, with particular interest in three variables: ${r^0}$, ${i^0}$ and ${u^0}$. However, considering the numerous combinations for selecting initial values for these three variables, it significantly increases the experimental cost.  In fact, we find that the three variables satisfy the relation $u = {e^{r + i}}$ by the numerical algorithm in this paper, so we only choose different initial values ${u^0}$ to verify their robustness to the experimental results. Base on this, we set four initial value images ${u^0} = zeros$ (all gray values are $0$), $ones$ (all gray values are $1$), $random$ (gray values are some random numbers between $0$ to $1$) and $f$ (observed image) as experimental objects.

Figure 3 illustrates the experimental results of the ``Cameraman'' image obtained by the proposed model using different initial values of ${u^0}$. The first row shows the difference image between the noisy input image and the restored image, which mainly consists of noise components and oscillation characteristics. The second and third rows display the restored image and a locally zoomed-in region of that image, respectively.
Upon examining these results, it is evident that the visual effects of the restored images are almost identical regardless of the different initial values ${u^0}$.

In addition, we also select two other test images, namely  ``Butterfly'' and ``Peppers1'', to serve as subjects for supplementary experiments, and test them at different noise levels. Table 1 shows the PSNR values of the restored image obtained by the proposed model with different initial values ${u^0}$. It should be noted that the last column in Table 1 represents the standard deviation (SD) of the four PSNR values in each row, which measures the dispersion of PSNR values. From the results, it is evident that these standard deviation (SD) values are consistently around 0.02, which indicates that the different initial values of ${u^0}$ have a negligible  quantitative impact on the experimental results.

\begin{center}  
	\begin{minipage}{\textwidth}  
		\centering  
		\includegraphics[width=\textwidth]{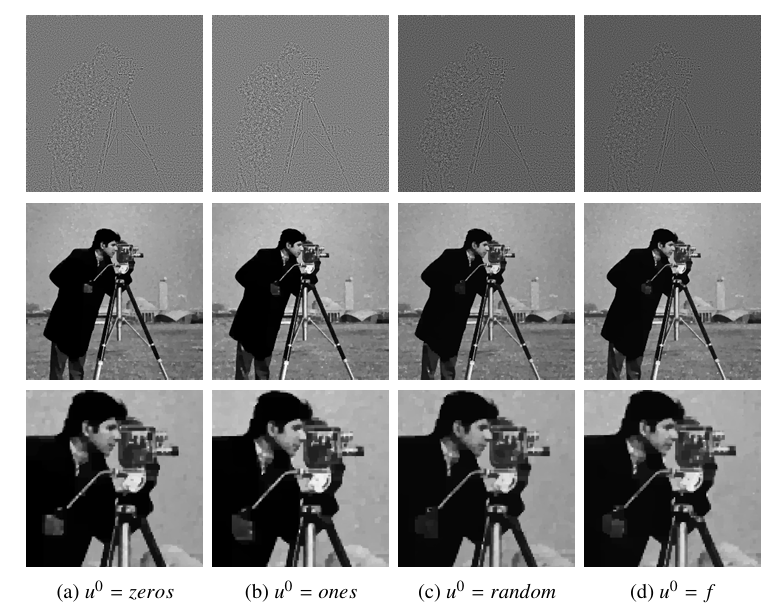}
		\vskip2mm  
		{\small \textbf{Fig.3} The experimental results of ``Cameraman'' image under different initial values ${u^0}$ with $\sigma = 15$.}  
	\end{minipage}  
\end{center}

\begin{table}[htbp]
	\centering
	\footnotesize
	{\small \textbf{Table 1.} The PSNR values of denoising for different initial values ${u^0}$.}
	\begin{tabular}{ccccccc}
		\toprule
		Image & Noise level &${u^0} = zeros$ &${u^0} = ones$ &${u^0} = random$ & ${u^0} = f$ & SD \\
		\midrule
		Butterfly & $\sigma  = {\rm{10}}$    & 32.1851 & 32.2216 & 32.1904 & 32.1744 & \textbf{0.0203} \\
		Peppers1 & $\sigma  = {\rm{10}}$    & 31.4034 & 31.3994 & 31.3599 & 31.3578 & \textbf{0.0246} \\
		Cameraman & $\sigma  = {\rm{10}}$    & 31.5367 &31.4948 & 31.4953 & 31.4858 & \textbf{0.0228} \\
		Cameraman & $\sigma  = {\rm{15}}$    & 29.3441 &29.3091 & 29.3620 & 29.3419& \textbf{0.0220} \\
		Cameraman & $\sigma  = {\rm{20}}$    & 27.7944 & 27.7576 & 27.7893 & 27.8157& \textbf{0.0240} \\
		\toprule
	\end{tabular}%
	\label{tab:addlabel}%
\end{table}%

\subsection{Test of different potential functions}

This example aims to examine the impact of different non-convex potential functions on the image restoration outcomes of the proposed model. To accomplish this, we have chosen two specific non-convex potential functions, namely ${\varphi _1}\left( \cdot \right)$ and ${\varphi _2}\left( \cdot \right)$, for conducting comparative experiments on the test image ``Peppers1'' contaminated with Gaussian noise having a variance of $\sigma=15$. The functional  graphs corresponding to the exponential function ${\varphi _1}\left( \cdot \right)\left( {p = 0.2,\;0.5,\;0.7,\;0.9} \right)$ and the logarithmic function ${\varphi _2}\left( \cdot \right)\left( {\alpha = 1,\;2,\;4,\;6} \right)$ are depicted in Fig.4.

It is important to note that the purpose of this experiment is primarily to assess the impact of different non-convex potential functions on the denoising results of the model rather than achieving optimal image restoration effects. Therefore, apart from varying the non-convex parameters $p$ and $\alpha$ in the potential functions, all other parameters remain unchanged across all experiments. Consequently, it is evident that the image restoration outcomes may not be optimal in this particular scenario.

It is widely recognized that non-convex functions can effectively preserve structural information such as image edges and contours, exhibiting strong denoising capabilities. However, they may yield staircase effects in smooth areas. From the function graph of the non-convex potential function shown in Fig.4, it is evident that the non-convex function ${\varphi _1}\left( \cdot \right)$ becomes weaker as the parameter $p$ increases. In other words, a smaller value of $p$ (e.g., $p = 0.2$) can better retain structural information in the image, while increasing the value of $p$ weakens the denoising effect but restores smoother areas.
On the contrary, the change trend of the parameter $\alpha$ in the function ${\varphi _2}\left( \cdot \right)$ is opposite to that of $p$, i.e., larger values of $\alpha$ result in stronger non-convexity. Fig.5 illustrates the schematic diagram of the proposed model after denoising the ``Peppers1'' image using different values of $p$ and $\alpha$ for the non-convex potential functions ${\varphi _1}\left( \cdot \right)$ and ${\varphi _2}\left( \cdot \right)$ respectively. By comparing Fig.5(i) and (l), it is visually evident that compared to when $p = 0.9$, a smaller value of $p=0.2$ leads to more pronounced staircase effects, which further confirms our previous observations based on visual effects. Similarly, a similar conclusion can be drawn by observing Fig.5(q) and (t).
Fig.6 presents local cross-sections of the denoised image, where we also observe that stronger non-convexity results in more severe staircase effects.

Furthermore, Table 2 presents the  PSNR  and  MSSIM  values for all denoised images in Fig.5, with the optimal value highlighted in bold. It is observed that the image quality is highest when $p = 0.2$ in the non-convex potential function ${\varphi _1}\left( \cdot \right)$. This suggests that the reflection component $R$ and illumination component $I$ of the ``Peppers1'' image have a sparse representation in the  TV  transform domain under this condition. In fact, when $\alpha = 2$ in ${\varphi _2}\left( \cdot \right)$, the non-convexity becomes stronger compared with the case of $\alpha = 1$, resulting in higher PSNR and MSSIM values for the restored image.
However, since the ``Peppers1'' image is not strictly sparse in the TV transform domain, as $\alpha$ continues to increase, the quality of the restored image decreases.

\begin{center}
	\begin{minipage}
		{1\textwidth}\centering {\includegraphics[width=12cm,height=4.5cm]{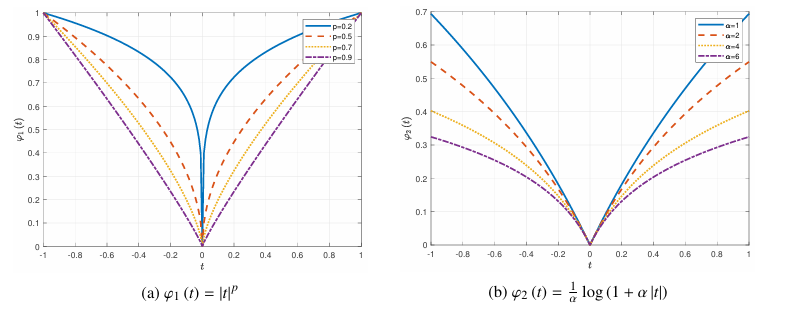}}
	\end{minipage}
	\vskip2mm
{\small \textbf{Fig.4}  Plots of different non-convex potential functions.} 
\end{center}

\begin{center}
	\begin{minipage}
		{1\textwidth}\centering {\includegraphics[width=13cm,height=13.5cm]{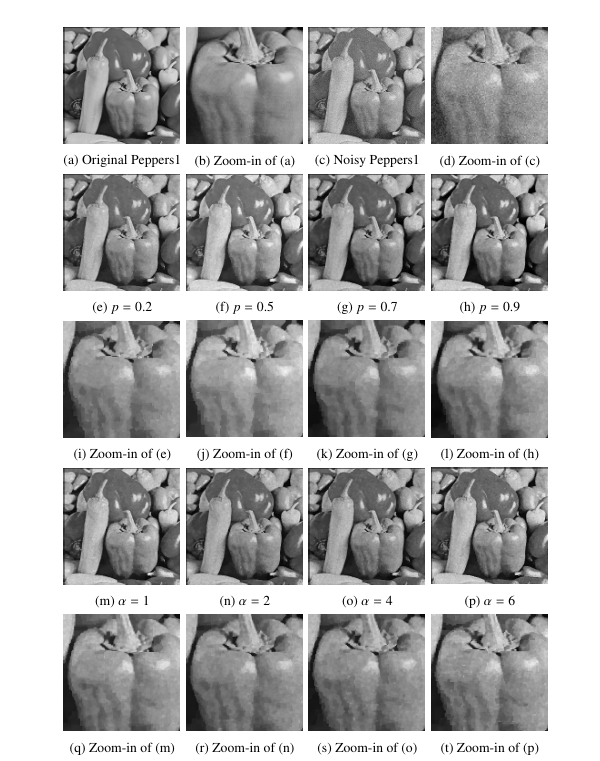}}
	\end{minipage}
	\vskip2mm
		{\small \textbf{Fig.5} Restoration results of ``Peppers1'' image using different potential functions with $\sigma=15$.}
\end{center}

\begin{table}[htbp]
	\centering
	\footnotesize
	{\small \textbf{Table 2.} The PSNR and MSSIM values of restored images in Fig.5.}
	\begin{tabular}{cccccc}
		\toprule
		& \multicolumn{2}{c}{${\varphi _1}\left( { t } \right)$} &       & \multicolumn{2}{c}{${\varphi _2}\left( { t } \right)$} \\
		\cmidrule{2-3}\cmidrule{5-6}          & PSNR  & MSSIM &       & PSNR  & MSSIM \\
		\midrule
		$p=0.2$ & \textbf{29.0718} & \textbf{0.9647} & $\alpha=1$ & 29.5156 & 0.9648 \\
		$p=0.5$ & 28.8074 & 0.9618 & $\alpha=2$ & \textbf{29.6635} & \textbf{0.9654} \\
		$p=0.7$ & 28.8694 & 0.9643 & $\alpha=4$ & 29.6343 & 0.9632 \\
		$p=0.9$ & 28.8407 & 0.9616 & $\alpha=6$ & 29.5064 & 0.9605 \\
		\bottomrule
	\end{tabular}%
	\label{tab:addlabel}%
\end{table}%

\begin{center}  
	\begin{minipage}{\textwidth}  
		\centering  
		\includegraphics[width=\textwidth]{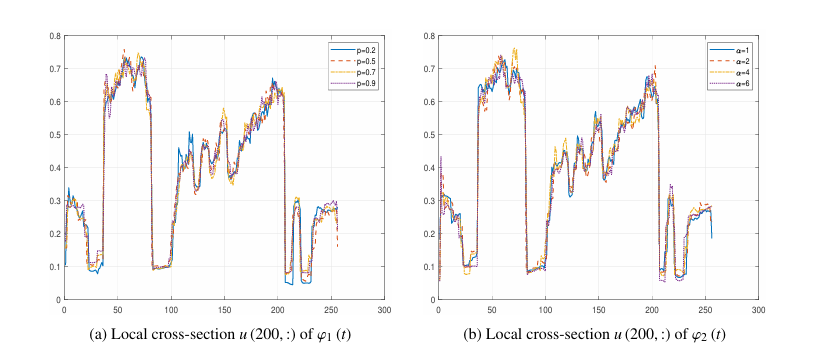}
		\vskip2mm  
		{\small \textbf{Fig.6}  The local cross-section of restored images in Fig.5.}
	\end{minipage}  
\end{center}

\subsection{Three component decomposition  by the proposed  model}

In this section, we select two typical test images ``Logvi'', and ``Checkboard'' to verify the effectiveness of the proposed model at different noise levels. In fact, the proposed model (3.2) in this paper is based on the retinex theory background, which decomposes a noisy image $f$ into three different components $f = I \odot R + n$ according to the retinex imaging theory (3.1), that is, the reflection component $R = {e^r}$ with piecewise constant characteristics, the illumination component $I = {e^i}$ with spatial smoothing characteristics, and the noise $n$ with high frequency oscillation characteristics. It is worth noting that the main purpose of this experiment is to verify that the proposed model can accurately decompose the above three components and can use $R$ and $I$ to reconstruct the restored image $u = I \odot R$.

Fig.7 shows the decomposed and restored image results for ``Logvi'' at $\sigma = 10$ and $\sigma = 20$,  with the last row representing the local cross sections $R\left( {200,:} \right)$, $I\left( {200,:} \right)$ and $u\left( {200,:} \right)$ from left to right. We find that the noisy image ``Logvi'' is well decomposed into piecewise constant component $R$, spatially smooth component $I$ and oscillatory component $n$ from Fig.7. In addition, the restored image $u$ reconstructed according to the reflectivity $R$ and the illumination $I$ also retains the edge and texture information, which indicates that the proposed decomposition model can be effectively applied to image denoising. For example, Fig.7(b)-(d) show a schematic diagram of the three components $R$, $I$ and $n$ when $\sigma = 10$, respectively, and (e) is the reconstructed image $u$ from $R$ and $I$. In order to illustrate the validity of the model more clearly, we may as well observe the Fig.7(g)-(j), which is the locally zoom-in region of above four subimages. Obviously, we can clearly find from the image that the reflection component $R$ has an obvious piecewise constant region, the illumination component $I$ has a spatial smoothing feature, and the noise component $n$ is chaotic. It can be seen from the Fig.7(j) that the restored image remains the edge and texture information better, and compared with the reflected image $R$, it has a slighter step effect and is more suitable the characteristics of a natural image.

We also test the decomposition effect of ``Checkboard'' image at $\sigma = 15$, and the decomposition results and cross-section diagram are shown in Fig.8. Obviously, we find from the Fig.8 that the proposed model divides the noise image into three components $R$, $I$ and $n$ with different characteristics, and successfully reconstructs the restored image $u$. The third row in Fig.8 shows the local cross-section diagrams of $R$, $I$ and $u$, respectively. By observing the gray value curve Fig.8(f)-(h), we notice that the reflection image $R$ contains a wide range of piecewise constant regions, and the illumination $I$ has the characteristics of piecewise linear smoothing. This is because the proposed model uses a generalized nonconvex first-order TV regularizer to measure the reflection component $R$, a generalized nonconvex second-order TV regularizer to measure the illumination component $I$, and the weaker-norm ${H^{ - 1}}$ to extract the oscillation component $n$.

From the above two decomposition experiments, we conclude that the proposed model can accurately extract the characteristics of different components in the degraded image, i.e., the reflectance $R$, illumination $I$ and noise $n$ can be well separated.

\begin{center}
	\begin{minipage}
		{1\textwidth}\centering {\includegraphics[width=12cm,height=12cm]{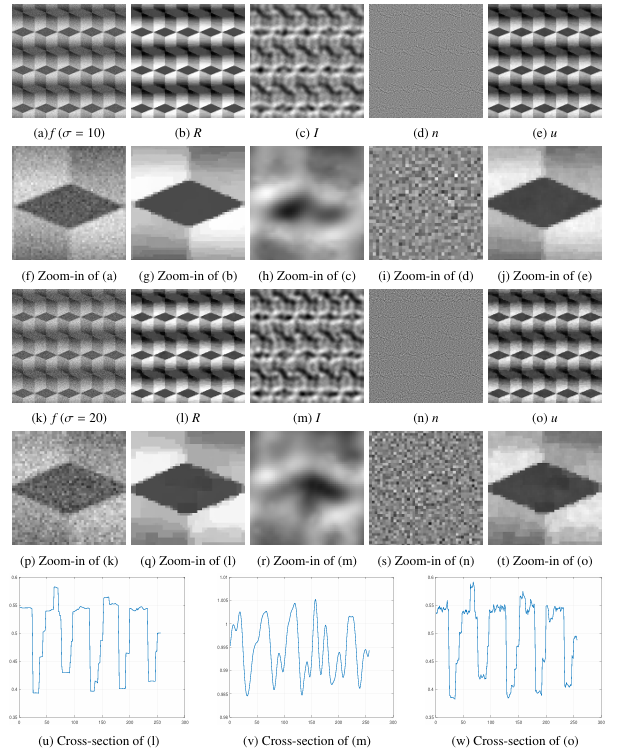}}
	\end{minipage}
	\vskip2mm
	{\small \textbf{Fig.7} Decomposition results of ``Logvi'' image at different noise levels.}
\end{center}

\begin{center}  
	\begin{minipage}{\textwidth}  
		\centering  
		\includegraphics[width=\textwidth]{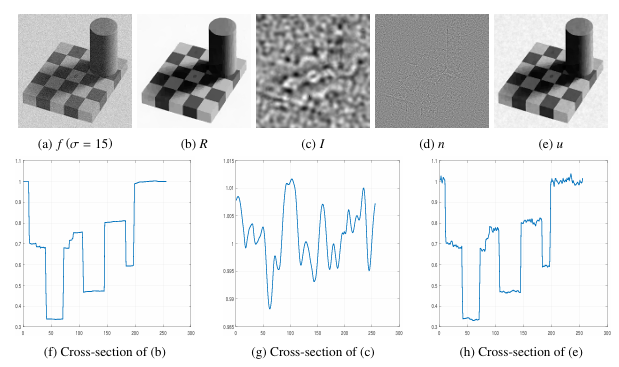}
		\vskip2mm  
		{\small \textbf{Fig.8} Decomposition results of ``Checkboard'' image with $\sigma=15$.}
	\end{minipage}  
\end{center}

\subsection{The proposed decomposition model for image denoising application}

The proposed variational decomposition model in this paper is based weak space and hybrid nonconvex regularization, so the expression of the restored image is described as $u = I \odot R = {e^{i + r}}$. In order to evaluate the performance of the proposed model in the field of image denoising, we compare the proposed model with some classical or popular variational models on different types of test images in this section. And the main purpose of the experiment is to verify the effectiveness of the denoising performance of the proposed model on various types of images. Specifically, we test the restoration effect of the proposed model on four image sets (gray images, real brain MR images, color images and three data sets), and compare some other variational models, including NTV \cite{b31}, TGV \cite{b38}, NNTV-$H^{-1}$ \cite{b34} and NETV \cite{b21} models.

As we know, NTV model is the most classical nonconvex variational model with excellent structural recovery efficiency. The TGV model is a high-order variational hybrid regularization model, which has excellent denoising effect but insufficient retention of structural information such as edges and contours. And we set the parameters ${\alpha _0} = 2$ and ${\alpha _1} = 1$ according to the suggestion of \cite{b38} in the experiment. NNTV-$H^{-1}$ model is a variational decomposition model proposed by Tang et al., which uses weak space and nonconvex regularizers to decompose images into oscillatory and structural components. Since the model directly regards the structural component as the restored image, the restored image lacks some texture oscillation information. NETV model is based on retinex theory and can better restore the image structure by using non-convex regularization term to model the source reflectance and illumination respectively. Finally, it should be noted that the above four variational models used for comparison are solved under the framework of ADMM algorithm, and the non-convex potential functions in the models are selected as $\varphi \left( { t } \right) = {\left| t \right|^p}$.

\noindent\textbf{Example 1.} Tests on gray images

In this example, we test the denoising performance of the proposed model on three gray images ``Cameraman'', ``Butterfly'', and ``Peppers1''. In order to make the proposed model more convincing in denoising effectiveness, we let these gray images used in the experiment be contaminated by Gaussian noise with different noise levels ($\sigma=10$ and 15). Fig.9 and Fig.10 show the denoising results of gray images contaminated by Gaussian noise with standard deviation $\sigma = 10$ and $\sigma = 15$, respectively. In which, the data in the second and fourth rows is the zoom-in region of the images, and the last row represents the colorbars of the difference image between the restored image and the real clean image. In addition, all the denoising data of the above gray test images are presented in Table 3, and the optimal values are bolded.

From Fig.9, we find that the NTV model and NNTV-$H^{-1}$ model are similar in the visual perception of the restored image, that is, they can better restore the sharp edge and contour information in the images. However, the smooth region of the image causes serious staircase effect, which destroys texture and detail information of the images. It is worth noting that the NNTV-$H^{-1}$ model suffers from a slighter staircase phenomenon than the NTV model, because the NNTV-$H^{-1}$ model uses weak space to separate the oscillating components in the images. TGV model has strong denoising efficiency and can restore smooth areas in images, but it also over-smooths the structural information such as edges and contours in images. For example, from the local zoom-in region of the ``Peppers1'' image in Fig.9, it is easy to find that the NTV model exists more serious piecewise constant region, while the TGV model blurs the edges of the image. Compared with the above models, NETV model has better edge and texture recovery effect, but its denoising performance is insufficient. Finally, we draw the  point from the difference image shown in figures that the proposed model has the singlest colorbars, that is, the proposed model has better results in both denoising performance and texture and edge retention.

In addition, the PSNR and MSSIM values shown in Table 4 can also intuitively find that the proposed model has the best PSNR value compared with the above variational models. It is undeniable that the MSSIM values of restored ``Cameraman'' image by the TGV model is slightly higher than those of the proposed model. Therefore, we further illustrate the superiority of the proposed model from a quantitative perspective based on the numerical results obtained from the experiment.

\begin{center}  
	\begin{minipage}{\textwidth}  
		\centering  
		\includegraphics[width=\textwidth]{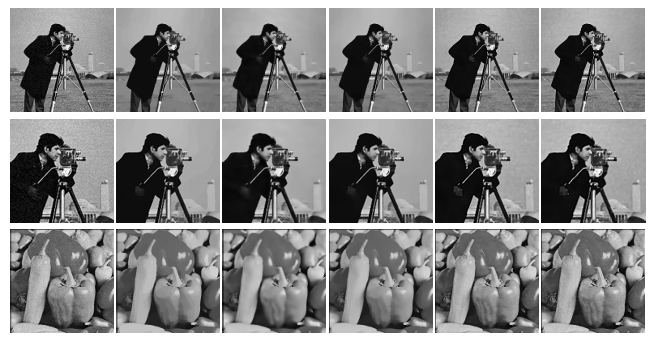}
		\vskip2mm  
		{\small \textbf{Fig.9} The denoising results of different models for the gray test images($\sigma  = {\rm{10}}$). }
	\end{minipage}  
\end{center}

\begin{center}
\begin{minipage}
	{1\textwidth}\centering {\includegraphics[width=13.4cm,height=11cm]{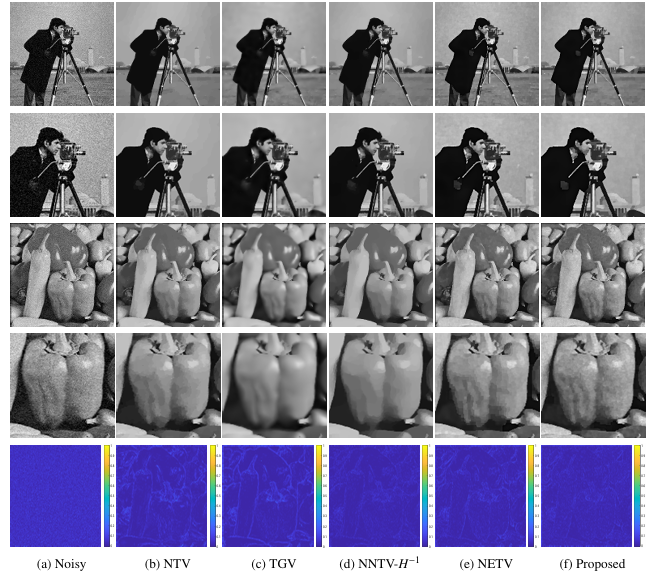}}
\end{minipage}
\vskip2mm
{\small \textbf{Fig.10} The denoising results of different models for the gray test images($\sigma  = {\rm{15}}$). }
\end{center}

\begin{table}[htbp]
	\centering
	\footnotesize
	{\small \textbf{Table 3.} The PSNR and MSSIM values for denoised gray images by different models.}
	\begin{tabular}{cccccccc}
		\toprule
		 $\sigma$   &  Images     & Noised & NTV   & TGV   & NNTV-$H^{-1}$ & NETV  & Proposed \\
		\midrule
		10    & Cameraman & 28.09/0.81 & 29.30/0.92 & 30.09/\textbf{0.94} & 29.97/0.92 & 29.89/0.93 & \textbf{31.70}/0.93 \\
		& Peppers1 & 28.13/0.93 & 29.83/0.96 & 30.34/0.97 & 30.81/0.97 & 31.29/0.97 & \textbf{31.81}/\textbf{0.98} \\
		& Butterfly & 28.15/0.93 & 30.58/0.96 & 31.47/0.98 & 31.63/0.97 & 31.82/0.98 & \textbf{32.54}/\textbf{0.98} \\
		\\
		15    & Cameraman & 24.57/0.72 & 28.45/0.91 & 28.72/\textbf{0.92} & 29.18/0.91 & 28.53/0.91 & \textbf{29.36}/0.91 \\
		& Peppers1 & 24.59/0.87 & 27.89/0.94 & 28.37/0.95 & 28.42/0.95 & 29.19/0.95 & \textbf{29.49}/\textbf{0.96} \\
		& Butterfly & 24.60/0.90 & 28.76/0.95 & 29.43/0.96 & 29.11/0.96 & 29.63/0.96 & \textbf{29.87}/\textbf{0.97} \\
		\bottomrule
	\end{tabular}%
	\label{tab:addlabel}%
\end{table}%

\noindent\textbf{Example 2.} Tests on real brain MR images

As we know, medical images are produced by complex imaging systems. During the imaging process, the images will be seriously degraded, which is caused by the interference of unstable signals such as electromagnetic or current. Therefore, in this example, we select three real cleaner brain MR images as test objects and add high intensity Gaussian noise with standard deviation $\sigma = 20$ to them in order to evaluate the denoising effect of the proposed model on complex medical images. Fig.11 and Fig.12 show the denoising results of ``Brain1'', ``Brain2'' and ``Brain3'' images respectively and contain the zoom-in region of the test images and the colorbars of the difference images. The PSNR and MSSIM values are listed in Table 4. In addition, we also show the gray value function curve of the local cross-section $u\left( {200,150:250} \right)$ of the restored images in Fig.13.

From the denoising results, we find that the NTV model exists serious piecewise constant regions, but this influence is smaller for the NNTV-$H^{-1}$ model. TGV model removes most of the noise in images. Such as the zoom-in region of ``Brain1'' in Fig.11, it is obvious that the noise in smooth regions almost non-existent, but the sharp contour may be mistakenly regarded as noise being smoothed. NETV model can balance the recovery of edge and smooth region, but it is difficult to separate high-frequency oscillation information, so its denoising effect is less effective. It is clear that the proposed model overcomes the above shortcomings, and can preserve the edge and texture information while denoising. In addition, we draw a similar conclusion by observing the PSNR and MSSIM values in Table 4, i.e., the proposed model has the largest PSNR and MSSIM values on real brain MR images.

Finally, we further verify the above viewpoint from the gray value curve of the local cross-section $u\left( {200,150:250} \right)$ of the restored image. From Fig.13(a), we find that the TGV model smoothes the region where the gray value changes rapidly, and NNTV-$H^{-1}$ model exists staircase effect in low-frequency oscillation regions. Similarly, from Fig.13(b), we also observe that the NTV model presents the more serious piecewise constant phenomenon, and the NETV model retains some high-frequency oscillation components. It can be concluded that the proposed model is superior and generalizable and shows excellent results in recovering complex medical images.

\begin{table}[htbp]
	\centering
	\footnotesize
	{\small \textbf{Table 4.} The PSNR and MSSIM values for denoised real brain MR images by different models.}
	\begin{tabular}{ccccccccc}
		\toprule
		& \multicolumn{2}{c}{Brain1} &       & \multicolumn{2}{c}{Brain2} &       & \multicolumn{2}{c}{Brain3} \\
		\cmidrule{2-3}\cmidrule{5-6}\cmidrule{8-9}          & PSNR  & MSSIM &       & PSNR  & MSSIM &       & PSNR  & MSSIM \\
		\midrule
		Noisy & 22.11 & 0.68  &       & 22.14 & 0.65  &       & 22.10  & 0.72 \\
		NTV   & 28.20  & 0.88  &       & 27.89 & 0.84  &       & 27.20  & 0.84 \\
		TGV   & 28.67 & 0.90   &       & 28.95 & 0.89  &       & 27.65 & 0.88 \\
		NNTVH-1 & 29.08 & 0.91  &       & 28.76 & 0.88  &       & 27.59 & 0.86 \\
		NETV  & 29.28 & 0.90   &       & 29.02 & 0.88  &       & 28.07 & 0.90 \\
		Proposed & \textbf{29.67} & \textbf{0.92}  &       & \textbf{29.38} & \textbf{0.90}   &       & \textbf{28.24 }& \textbf{0.90} \\
		\bottomrule
	\end{tabular}%
	\label{tab:addlabel}%
\end{table}%

\begin{center}  
	\begin{minipage}{\textwidth}  
		\centering  
		\includegraphics[width=\textwidth]{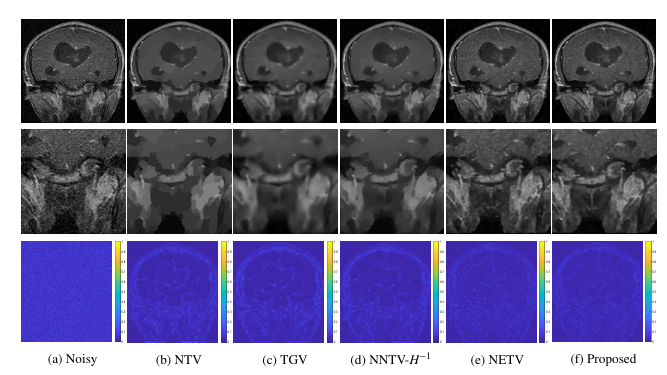}
		\vskip2mm  
		{\small \textbf{Fig.11} The denoising results of noisy ``Brain1'' image with $\sigma=20$  }
	\end{minipage}  
\end{center}

	\begin{center}  
		\begin{minipage}{\textwidth}  
			\centering  
			\includegraphics[width=\textwidth]{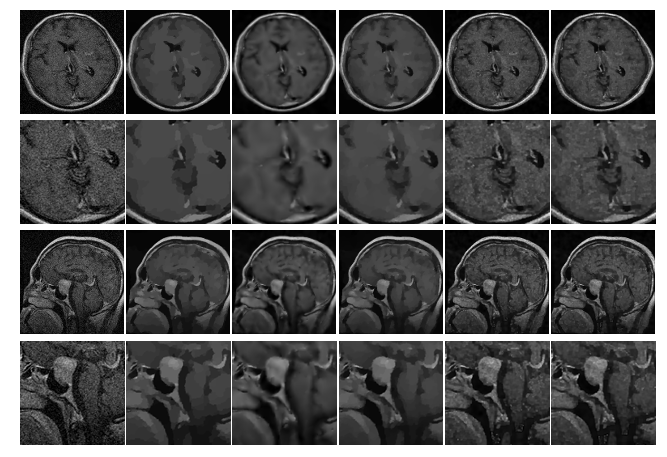}
			\vskip2mm  
			{\small \textbf{Fig.12} The denoising results of noisy ``Brain2'' and ``Brain3'' images with $\sigma  = 20$. }
		\end{minipage}  	
\end{center}

	\begin{center}  
	\begin{minipage}{\textwidth}  
		\centering  
		\includegraphics[width=\textwidth]{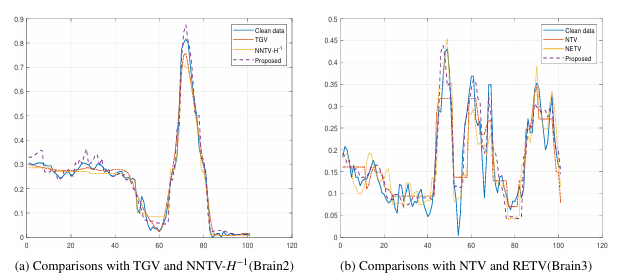}
		\vskip2mm  
		{\small \textbf{Fig.13}  The local cross-section $u\left( {200,150:250} \right)$ of restored images in Fig.12.}
	\end{minipage}  	
\end{center}

\noindent\textbf{Example 3.} Tests on color images

In this example, we test the denoising performance of the proposed model on some color images. In fact, a color image contains three RGB color channels, then it can also be considered as a three-dimensional array of $M \times N \times 3$. And the noisy color images are generated by the RGB channels that are all contaminated by Gaussian noise. In this experiment, it should be noted that all the denoising models denoise the three color channels of the image, and the PSNR and MSSIM values are the average values of the three restored RGB channels. The PSNR and MSSIM values of the denoised image are listed in Table 5, and Fig.14 shows the recovery results of the color images at $\sigma = 15$, where the third and fifth rows represent the local zoom-in region of the images.

Similarly, from Fig.14, we clearly observe that the NTV model and NNTV-$H^{-1}$ model recover the structural information of the image well. For example, the overall structure of the house in ``House'' image is well preserved, but the smooth information of the roof and wall has a serious staircase effect. It is worth noting that the TGV model can restore these smooth regions mentioned above, but its strong denoising performance will cause the loss of edge information. NETV model can not only retain the contour of the image, but also restore part of the texture and detail information, but its denoising performance is insufficient to completely eliminate the noise in images. The proposed model combines the advantages of the above models, i.e., this model effectively reduces the staircase phenomenon, and can remove the noise while restoring the image edge and texture information.

In addition, the experimental data in Table 5 also reflect the advantages of the proposed model in color image restoration. We find that the MSSIM values of NNTV-$H^{-1}$ model are only slightly higher than those of the proposed model for ``Boats'' image at $\sigma=15$. In all the remaining cases, the PSNR and MSSIM values of the proposed model are ahead of other variational models. In summary, we can conclude that the proposed model still has superior recovery results on color images.

\begin{table}[htbp]
	\centering
	\footnotesize
	{\small \textbf{Table 5.} The PSNR and MSSIM values for denoised color images by different models.}
	\begin{tabular}{cccccccc}
		\toprule
		$\sigma$ & Images      & Noised & NTV   & TGV   & NNTV-$H^{-1}$ & NETV  & Proposed \\
		\midrule
		15   & Boats  & 24.62/0.86 & 26.56/0.91 & 27.17/0.93 & 27.49/\textbf{0.96} & 27.94/0.93 & \textbf{28.26}/0.94 \\
		& House & 24.60/0.67 & 29.60/0.85 & 30.74/0.89 & 30.32/0.87 & 30.24/0.85 & \textbf{30.86}/\textbf{0.90} \\
		& Peppers2 & 24.61/0.82 & 27.37/0.89 & 29.01/0.94 & 29.01/0.93 & 29.24/0.91 & \textbf{29.45}/\textbf{0.94} \\
		\\
		20    & Boats  & 22.12/0.79 & 25.04/0.88 & 25.78/0.90 & 26.23/0.93 & 26.54/0.91 & \textbf{26.73}/\textbf{0.92} \\
		& House & 22.12/0.58 & 28.42/0.82 & 29.75/0.85 & 29.52/0.84 & 29.42/0.83 & \textbf{29.92}/\textbf{0.86} \\	
		& Pepers2 & 22.12/0.74 & 26.22/0.86 & 27.64/0.88 & 27.71/0.91 & 27.87/0.90 & \textbf{28.09}/\textbf{0.91} \\
		\bottomrule
	\end{tabular}%
	\label{tab:addlabel}%
\end{table}%

		\begin{center}  
		\begin{minipage}{\textwidth}  
			\centering  
			\includegraphics[width=\textwidth]{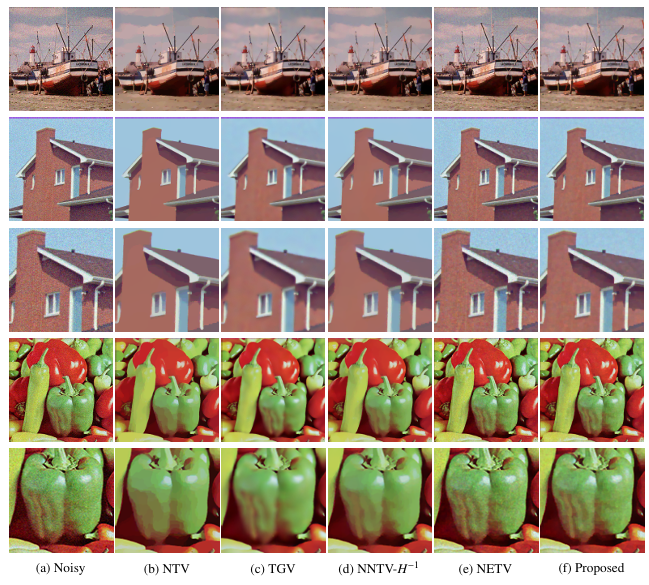}
			\vskip2mm  
			{\small \textbf{Fig.14} The denoising results of the color test images with $\sigma  = 15$. }
		\end{minipage}  	
	\end{center}

\noindent\textbf{Example 4.} Tests on three DataSets

In the last example, we further validate the effectiveness and superiority of the proposed model by conducting experiments on three image datasets: Set14 \cite{b39}, IVC \cite{b40}, and Tid2008 \cite{b41}. The proposed model is compared with several classical or state-of-the-art  variational models in PSNR values, including TV \cite{b42}, NTV \cite{b31}, TGV \cite{b38}, HOTV \cite{b43}, NLTV \cite{b44}, OGS-TV \cite{b45}, NNTV-$H^{-1}$ \cite{b34} and NETV \cite{b21}. It should be noted that we do not impose any mandatory resizing of test images in these datasets.

Table 6 lists the average PSNR values of the restored images obtained by the denoising models mentioned above on three datasets with different noise levels. Based on the numerical results, it is evident that the proposed model achieves the highest average PSNR values, indicating superior recovery performance. Thus far, in this subsection, we have effectively demonstrated the generalization and superiority of our proposed model in image restoration through four illustrative examples.

\begin{table}[htbp]
	\centering
	\footnotesize
	{\small \textbf{Table 6.} The average PSNR of the restored images form Set14, IVC and Tid2008.}
	\begin{tabular}{cccccccccccc}
		\toprule
		& \multicolumn{3}{c}{$\sigma=10$} &       & \multicolumn{3}{c}{$\sigma=15$} &       & \multicolumn{3}{c}{$\sigma=20$} \\
		\cmidrule{2-4}\cmidrule{6-8}\cmidrule{10-12}          & Set14 & IVC   & Tid2008 &       & Set14 & IVC   & Tid2008 &       & Set14 & IVC   & Tid2008 \\
		\midrule
		Noised & 28.13 & 28.13 & 28.12 &       & 25.23 & 25.24 & 25.23 &       & 22.57 & 22.56 & 22.56 \\
		TV    & 32.73 & 32.73 & 32.72 &       & 29.18 & 29.14 & 29.13 &       & 27.21 & 27.18 & 27.15 \\
		NTV   & 32.44 & 32.42 & 32.40 &       & 29.02 & 28.95 & 28.92 &       & 27.14 & 27.02 & 27.00 \\
		TGV   & 32.94 & 32.89 & 32.84 &       & 29.75 & 29.54 & 29.46 &       & 27.73 & 27.46 & 27.32 \\
		HOTV  & 32.59 & 32.57 & 32.56 &       & 29.12 & 29.07 & 28.95 &       & 27.18 & 27.16 & 27.12 \\
		NLTV  & 32.97 & 32.92 & 32.88 &       & 29.82 & 29.64 & 29.50 &       & 27.88 & 27.63 & 27.52 \\
		OGS-TV & 33.05 & 32.92 & 32.90 &       & 29.85 & 29.68 & 29.55 &       & 28.04 & 27.89 & 27.90 \\
		NNTV-H-1 & 32.89 & 32.87 & 32.83 &       & 29.37 & 29.16 & 29.04 &       & 27.43 & 27.25 & 27.14 \\
		NETV  & 33.22 & 33.05 & 32.95 &       & 30.54 & 30.25 & 30.11 &       & 28.29 & 28.03 & 27.95 \\
		Proposed & \textbf{33.62} & \textbf{33.48} & \textbf{33.32} &       & \textbf{31.12} & \textbf{30.95} & \textbf{30.78} &       & \textbf{28.75} & \textbf{28.61} & \textbf{28.58} \\
		\bottomrule
	\end{tabular}%
	\label{tab:addlabel}%
\end{table}%

\section{Conclusions}

In this paper, we propose an exponential Retinex decomposition model for image denoising based on weak space and hybrid nonconvex regularization. This model decomposes the noisy image into three incoherent parts, and reconstructs the denoised image by using the reflection component and the illumination component. Specifically, the oscillation component, reflection component and illumination component are measured by weak $H^{-1}$ space, nonconvex first order TV regularizer and nonconvex second order TV regularizer, respectively. In addition, we also propose an ADMM combined with MM algorithm to solve the proposed model, and provide sufficient conditions for the convergence of the proposed algorithm. Numerical experiments were conducted to compare the proposed model with several state-of-the-art denoising models. The experimental results demonstrate that the proposed model outperforms these models in terms of PSNR and MSSIM values. This further confirms its effectiveness and superiority.

It is important to acknowledge that there are still some limitations in the proposed model, which also serve as potential directions for our future research: (1) There is a need to explore more efficient optimization algorithms that can guarantee the uniqueness and global convergence of the numerical solution. This will enhance the reliability and stability of the model; (2) The proposed model takes too much time by using the parameter values determined by trial and error method. In the future work, we will focus on the adaptive methods of these parameters.

\section*{Acknowledgement}

This work was supported in part by the Natural Science Foundation of China under Grant No. 62061016, 61561019, the Doctoral Scientific Fund Project of Hubei Minzu University under Grant No. MY2015B001, the Innovative Project of the School of Mathematics and Statistics, Hubei Minzu University under Grant No. STK2023002.



%
\section*{Declarations}

\noindent Conflict of Interest: The authors declare that they have no conflict of interest.

\vspace{1ex}\noindent Data Availability Statement: Data sharing not applicable to this article as no datasets were generated or analyzed during the current study.
%

\bibliographystyle{elsarticle-num}
\bibliography{mybibfile}

\begin{thebibliography}{10}
\expandafter\ifx\csname url\endcsname\relax
  \def\url#1{\texttt{#1}}\fi
\expandafter\ifx\csname urlprefix\endcsname\relax\def\urlprefix{URL }\fi
\expandafter\ifx\csname href\endcsname\relax
  \def\href#1#2{#2} \def\path#1{#1}\fi

\bibitem{b1}
C.~Maria, T.~Maciej, Z.~Chao, P.~Krzysztof, Enhancing single-shot fringe
  pattern phase demodulation using advanced variational image decomposition,
  Journal of Optics 21~(4) (2019) 045702.

\bibitem{b2}
S.~Wang, K.~Xia, L.~Wang, J.~Zhang, H.~Yang, Improved rpca method via weighted
  non-convex regularization for image denoising, IET Signal Processing 14~(5)
  (2020) 269--277.

\bibitem{b3}
L.~Fan, H.~Li, M.~Shi, Z.~Hua, C.~Zhang, Two-stage image denoising via an
  enhanced low-rank prior, Journal of Scientific Computing 90 (2022) 57.

\bibitem{b4}
L.~M. Tang, Y.~J. Ren, Z.~Fang, C.~J. He, A generalized hybrid nonconvex
  variational regularization model for staircase reduction in image
  restoration, Neurocomputing 359~(24) (2019) 15--31.

\bibitem{b5}
X.~G. Lv, Y.~Z. Song, S.~X. Wang, J.~Le, Image restoration with a high-order
  total variation minimization method, Applied Mathematical Modelling
  37~(16-17) (2013) 8210--8224.

\bibitem{b6}
J.~Bai, X.~C. Feng, Fractional-order anisotropic diffusion for image denoising,
  IEEE Transactions on Image Processing 16~(10) (2007) 2492--2502.

\bibitem{b7}
J.~F. Aujol, G.~Gilboa, T.~Chan, S.~Osher, Structure-texture image
  decomposition-modeling, algorithms, and parameter selection, International
  Journal of Computer Vision 67~(1) (2006) 111--136.

\bibitem{b8}
L.~M. Tang, Z.~Fang, C.~C. Xiang, S.~Q. Chen, Image selective restoration using
  multi-scale variational decomposition, Journal of Visual Communication and
  Image Representation 40~(pt.B) (2016) 638--655.

\bibitem{b9}
T.~F. Chan, S.~Esedoglu, F.~E. Park, Image decomposition combining staircase
  reduction and texture extraction, Journal of Visual Communication and Image
  Representation 18~(6) (2007) 464--486.

\bibitem{b10}
J.~M. Morel, A.~B. Petro, C.~Sbert, A pde formalization of retinex theory, IEEE
  Transactions on Image Processing 19~(11) (2010) 2825--2837.

\bibitem{b11}
X.~Fu, Y.~Liao, D.~Zeng, Y.~Huang, X.~Zhang, X.~Ding, A probabilistic method
  for image enhancement with simultaneous illumination and reflectance
  estimation, IEEE Transactions on Image Processing A Publication of the IEEE
  Signal Processing Society 24~(12) (2015) 4965.

\bibitem{b12}
W.~Ma, S.~Osher, A tv bregman iterative model of retinex theory, Inverse
  Problems and Imaging 6~(4) (2007) 697--708.

\bibitem{b13}
E.~H. Land, J.~J. McCann, Lightness and retinex theory, Journal of the Optical
  Society of America 61~(1) (1971) 1--11.

\bibitem{b14}
E.~H. Land, H.~Edwin, The retinex theory of color vision, Scientific American
  237~(6) (1978) 108--128.

\bibitem{b15}
Q.~Zhao, P.~Tan, Q.~Dai, L.~Shen, E.~Wu, S.~Lin, A closedform solution to
  retinex with nonlocal texture constraints, IEEE Transactions on Pattern
  Analysis and Machine Intelligence 34~(7) (2012) 1437--1444.

\bibitem{b16}
R.~Kimmel, M.~Elad, D.~Shaked, R.~Keshet, I.~Sobel, A variational framework for
  retinex, International Journal of Computer Vision 52~(1) (2003) 7--23.

\bibitem{b17}
M.~K. Ng, W.~Wang, A total variation model for retinex, SIAM Journal on Imaging
  Sciences 4~(1) (2011) 345--365.

\bibitem{b18}
J.~Liang, X.~Zhang, Retinex by higher order total variation ${L^{\rm{1}}}$
  decomposition, Journal of Mathematical Imaging and Vision 52~(3) (2015)
  345--355.

\bibitem{b19}
J.~Xu, Y.~Hou, D.~Ren, L.~Liu, F.~Zhu, M.~Yu, H.~Wang, L.~Shao, Star: A
  structure and texture aware retinex model, IEEE Transactions on Image
  Processing 29 (2020) 5022--5037.

\bibitem{b20}
L.~Liu, Z.~F. Pang, Y.~Duan, Retinex based on exponent-type total variation
  scheme, Inverse Problems and Imaging 12~(5) (2018) 1199--1217.

\bibitem{b21}
Y.~Wang, Z.~F. Pang, Y.~Duan, K.~Chen, Image retinex based on the nonconvex
  tv-type regularization, Inverse Problems and Imaging 15~(6) (2021)
  1381--1407.

\bibitem{b22}
Y.~Meyer, Oscillating Patterns in Image Processing and Nonlinear Evolution
  Equations: The Fifteenth Dean Jacqueline B. Lewis Memorial Lectures, American
  Mathematical Society, 2001.

\bibitem{b23}
L.~A. Vese, S.~J. Osher, Image denoising and decomposition with total variation
  minimization and oscillatory functions, Journal of Mathematical Imaging and
  Vision 20~(1-2) (2004) 7--18.

\bibitem{b24}
L.~M. Tang, L.~Wu, Z.~Fang, C.~Y. Li, A non‐convex ternary variational
  decomposition and its application for image denoising, IET signal processing
  16~(3) (2022) 248--266.

\bibitem{b25}
L.~M. Tang, C.~J. He, Multiscale texture extraction with hierarchical
  (bv,gp,l2) decomposition, Journal of Mathematical Imaging and Vision 45~(2)
  (2013) 148--163.

\bibitem{b26}
S.~Osher, A.~Sole, L.~Vese, Image decomposition and restoration using total
  variation minimization and the ${h^{ - 1}}$ norm, SIAM Multiscale Modeling
  and Simulation 1~(3) (2003) 349--370.

\bibitem{b27}
W.~Deng, W.~Yin, On the global and linear convergence of the generalized
  alternating direction method of multipliers, Journal of Scientific Computing
  66~(3) (2016) 889--916.

\bibitem{b28}
M.~V. Afonso, J.~M. Bioucas-Dias, M.~A.~T. Figueiredo, An augmented lagrangian
  approach to the constrained optimization formulation of imaging inverse
  problems, IEEE Transactions on Image Processing 20~(3) (2011) 681--695.

\bibitem{b29}
D.~R. Hunter, K.~Lange, A tutorial on mm algorithms, American Statistician
  58~(1) (2004) 30--37.

\bibitem{b30}
D.~R. Hunter, R.~Li, Variable selection using mm algorithms, Annals of
  Statistics 33~(4) (2005) 1617--1642.

\bibitem{b31}
M.~Nikolova, M.~K. Ng, C.~P. Tam, Fast nonconvex nonsmooth minimization methods
  for image restoration and reconstruction, IEEE Transactions on Image
  Processing 19~(12) (2010) 3073--3088.

\bibitem{b32}
M.~Nikolova, M.~K. Ng, S.~Zhang, W.~K. Ching, Efficient reconstruction of
  piecewise constant images using nonsmooth nonconvex minimization, SIAM
  Journal on Imaging Sciences 1~(1) (2008) 2--25.

\bibitem{b33}
Y.~Sun, P.~Babu, D.~P. Palomar, Majorization-minimization algorithms in signal
  processing, communications, and machine learning, IEEE Transactions on Signal
  Processing 65~(3) (2017) 794--816.

\bibitem{b34}
L.~M. Tang, H.~L. Zhang, C.~J. He, Z.~Fang, Non-convex and non-smooth
  variational decomposition for image restoration, Applied Mathematical
  Modelling 69 (2019) 355--377.

\bibitem{b35}
M.~Benning, F.~Knoll, C.-B. Schönlieb, T.~Valkonen, Preconditioned admm with
  nonlinear operator constraint, in: IFIP Conference on System Modeling and
  Optimization, 2015, pp. 117--126.

\bibitem{b36}
A.~Horé, D.~Ziou, Image quality metrics: Psnr vs. ssim, in: 2010 20th
  International Conference on Pattern Recognition, 2010, pp. 2366--2369.

\bibitem{b37}
W.~Zhou, A.~C. Bovik, H.~R. Sheikh, E.~P. Simoncelli, Image quality assessment:
  from error visibility to structural similarity, IEEE Transactions on Image
  Processing 13~(4) (2004) 600--612.

\bibitem{b38}
K.~Bredies, K.~Kunisch, T.~Pock, Total generalized variation, SIAM Journal on
  Imaging Sciences 3~(3) (2010) 492--526.

\bibitem{b39}
L.~Wu, L.~M. Tang, L.~C. Yan, Hybrid regularization model combining overlapping
  group sparse second-order total variation and nonconvex total variation,
  Journal of Electronic Imaging 31~(4) (2022) 043012.

\bibitem{b40}
P.~L. Callet, F.~Autrusseau, Subjective quality assessment irccyn/ivc database,
  Informatique Traitement Signal Image (2005).

\bibitem{b41}
N.~Ponomarenko, V.~Lukin, A.~Zelensky, K.~Egiazarian, M.~Carli, F.~Battisti,
  Tid2008-a database for evaluation of full-reference visual quality assessment
  metrics, Advances of Modern Radioelectronics 10~(4) (2009) 30--45.

\bibitem{b42}
L.~I. Rudin, S.~Osher, E.~Fatemi, Nonlinear total variation based noise removal
  algorithms, Physica D Nonlinear Phenomena 60~(1-4) (1992) 259--268.

\bibitem{b43}
X.~G. Lv, Y.~Z. Song, S.~X. Wang, L.~Jiang, Image restoration with a high-order
  total variation minimization method, Applied Mathematical Modelling
  37~(16-17) (2013) 8210--8224.

\bibitem{b44}
D.~Lv, Q.~Zhou, J.~K. Choi, J.~Li, X.~Zhang, Nonlocal tv-gaussian prior for
  bayesian inverse problems with applications to limited ct reconstruction,
  Inverse Problems and Imaging 14~(1) (2020) 117.

\bibitem{b45}
J.~Liu, T.~Z. Huang, I.~W. Selesnick, X.~G. Lv, P.~Y. Chen, Image restoration
  using total variation with overlapping group sparsity, Information Sciences
  295 (2015) 232--246.

\end{thebibliography}

%
%

\end{document}